\documentclass[aos]{imsart}

\RequirePackage{amsthm,amsmath,amsfonts,amssymb}
\RequirePackage[authoryear]{natbib}
\RequirePackage[colorlinks,citecolor=blue,urlcolor=blue]{hyperref}
\RequirePackage{graphicx}
\RequirePackage{booktabs}
\RequirePackage[shortlabels]{enumitem}
\RequirePackage{bbm}
\RequirePackage{xcolor}

\startlocaldefs
\numberwithin{equation}{section}

\theoremstyle{plain}
\newtheorem{theorem}{Theorem}[section]
\newtheorem{lemma}[theorem]{Lemma}
\newtheorem{proposition}[theorem]{Proposition}
\newtheorem{corollary}[theorem]{Corollary}

\theoremstyle{definition}

\newtheorem{remark}[theorem]{Remark}

\newcommand{\E}{\mathbb{E}}

\newcommand{\R}{\mathbb{R}}
\newcommand{\norm}[1]{\left\|#1\right\|}
\newcommand{\opnorm}[1]{\left\|#1\right\|_{\mathrm{op}}}
\newcommand{\xbar}{\bar{x}_n}
\newcommand{\Sighat}{\widehat{\Sigma}_n}
\DeclareMathOperator{\tr}{tr}

\endlocaldefs

\begin{document}

\begin{frontmatter}

\title{Online Covariance Estimation in Averaged SGD:
Improved Batch-Means Rates and Minimax Optimality via
Trajectory Regression}
\runtitle{Online Covariance Estimation in Averaged SGD}

\begin{aug}
\author[A]{\fnms{Yijin}~\snm{Ni}\ead[label=e1]{yni64@gatech.edu}}
\author[A]{\fnms{Xiaoming}~\snm{Huo}\ead[label=e2]{xiaoming@isye.gatech.edu}}

\address[A]{H.~Milton Stewart School of Industrial and Systems Engineering,
Georgia Institute of Technology\printead[presep={,\ }]{e1,e2}}
\end{aug}

\begin{abstract}
We study online covariance matrix estimation for Polyak--Ruppert averaged
stochastic gradient descent (SGD).  The online batch-means estimator of
\citet{zhu2023online} achieves an operator-norm convergence rate of
$O\!\left(n^{-(1-\alpha)/4}\right)$, which yields $O(n^{-1/8})$ at the
optimal learning-rate exponent $\alpha\to 1/2^{+}$.  A rigorous
per-block bias analysis reveals that re-tuning the block-growth parameter
improves the batch-means rate to
$O\!\left(n^{-(1-\alpha)/3}\right)$, achieving $O(n^{-1/6})$.
The modified estimator requires no Hessian access and preserves $O(d^{2})$
memory.  We provide a complete error decomposition into variance,
stationarity bias, and nonlinearity bias components.
A weighted-averaging variant that avoids hard truncation is also discussed.
We establish the minimax rate $\Theta(n^{-(1-\alpha)/2})$ for Hessian-free
covariance estimation from the SGD trajectory: a Le~Cam lower bound gives
$\Omega(n^{-(1-\alpha)/2})$, and a trajectory-regression estimator---which
estimates the Hessian by regressing SGD increments on iterates---achieves
$O(n^{-(1-\alpha)/2})$, matching the lower bound.
The construction reveals that the bottleneck is the sublinear
accumulation of information about the Hessian from the SGD drift.
\end{abstract}



\end{frontmatter}

\section{Introduction}\label{sec:intro}

Stochastic gradient descent (SGD) and its variants are the workhorses of
modern machine learning and large-scale optimization.  Given a population
risk $F(x)=\E[f(x,\zeta)]$ with minimizer $x^{*}$, the Polyak--Ruppert
averaged iterate $\xbar=(1/n)\sum_{t=1}^{n}x_t$ achieves the optimal
$O(1/n)$ rate for the mean-squared error under standard regularity
conditions \citep{polyak1992,ruppert1988}.
Asymptotically,
\begin{equation}\label{eq:clt}
\sqrt{n}\,(\xbar - x^{*})\;\xrightarrow{d}\;
\mathcal{N}\!\left(0,\,V\right),
\qquad
V \;=\; H^{-1}S\,H^{-1},
\end{equation}
where $H=\nabla^{2}F(x^{*})$ is the Hessian at the optimum and
$S=\E[\nabla f(x^{*},\zeta)\,\nabla f(x^{*},\zeta)^{\top}]$ is the noise
covariance.  The matrix $V$ is the \emph{asymptotic covariance} and is
the key quantity needed to construct confidence intervals, hypothesis tests,
and other inferential statements about $x^{*}$.

A fundamental practical difficulty is that $V$ involves the Hessian $H$,
which is often expensive or impossible to compute.  In many modern
applications---deep learning, reinforcement learning, large-scale
regression---only first-order gradient information is available.  This has
motivated a line of research on \emph{Hessian-free} covariance estimation,
where one estimates $V$ directly from the SGD iterates $(x_1,\ldots,x_n)$
without ever forming or inverting $H$.

The \emph{online batch-means} estimator of \citet{zhu2023online} is an
elegant approach in this vein.  The iterate sequence is partitioned into
blocks of growing sizes; within each block, a centered sample covariance is
computed; the overall estimator is the average of these block-level
covariance matrices.  Because the blocks grow, the estimator can be updated
online as new iterates arrive, using only $O(d^{2})$ memory.
\citet{zhu2023online} showed that their estimator achieves operator-norm
convergence
\begin{equation}\label{eq:zhu-rate}
\E\!\left[\opnorm{\Sighat(1) - V}\right]
\;=\; O\!\left(n^{-(1-\alpha)/4}\right),
\end{equation}
where $\alpha\in(1/2,1)$ is the learning-rate exponent.  The best rate is
$O(n^{-1/8})$, obtained as $\alpha\to 1/2^{+}$.

\paragraph{Contributions.}
We introduce a \emph{burn-in parameter} $\rho\in(0,1]$ and a rigorous
per-block bias analysis into the batch-means framework.  The analysis
shows that the per-block stationarity bias is $O(\tau_m/a_m)$, where
$\tau_m/a_m$ is the ratio of mixing time to block size.  This leads to an improved batch-means
rate via optimal block-growth tuning:
\begin{equation}\label{eq:our-rate}
\E\!\left[\opnorm{\Sighat(\rho) - V}\right]
\;=\; O\!\left(n^{-(1-\alpha)/3}\right),
\end{equation}
yielding $O(n^{-1/6})$ as $\alpha\to 1/2^{+}$, an improvement over
the $O(n^{-1/8})$ rate of \citet{zhu2023online}.
We further establish the minimax rate $\Theta(n^{-(1-\alpha)/2})$
by constructing a trajectory-regression estimator that achieves
$O(n^{-(1-\alpha)/2})$, matching a Le~Cam lower bound.
The specific contributions are as follows.
\begin{enumerate}[(1)]
\item \textbf{Unified framework.}  We introduce a burn-in parameter
  $\rho\in(0,1]$ into the batch-means estimator that controls the
  bias--variance tradeoff (Section~\ref{sec:bm}).  The original estimator
  of \citet{zhu2023online} corresponds to $\rho=1$; the burn-in
  modification corresponds to $\rho<1$.
\item \textbf{Complete error analysis.}  We provide a three-way error
  decomposition---variance, stationarity bias, nonlinearity bias---and
  derive the batch-means rate $O(n^{-(1-\alpha)/3})$ via optimal
  block-growth tuning (Corollary~\ref{cor:main}).
\item \textbf{Minimax rate.}  We prove a Le~Cam lower bound
  $\Omega(n^{-(1-\alpha)/2})$ (Theorem~\ref{thm:lower}) and a matching
  upper bound via a trajectory-regression estimator
  (Theorem~\ref{thm:regression}), establishing the minimax rate
  $\Theta(n^{-(1-\alpha)/2})$ for Hessian-free covariance estimation.
\item \textbf{Weighted variant.}  We show that a soft weighting scheme
  $w_m\propto m^{p}$ achieves the same batch-means rate without hard
  truncation (Proposition~\ref{prop:weighted}).
\end{enumerate}

\paragraph{Rate comparison.}
Table~\ref{tab:rates} summarizes the landscape of convergence rates for
covariance estimation in the SGD setting.

\begin{table}[tbp]
\centering
\caption{Convergence rates for estimating $V$ in operator norm.
Here $\alpha\in(1/2,1)$ is the learning-rate exponent.  ``Hessian?''
indicates whether the method requires access to $H$ or $H^{-1}$.}
\label{tab:rates}
\smallskip
\footnotesize
\setlength{\tabcolsep}{4pt}
\begin{tabular}{@{}lcccc@{}}
\toprule
\textbf{Method} & \textbf{Rate} & \textbf{Hessian?} & \textbf{Online?}
& \textbf{Reference} \\
\midrule
Online BM &
$n^{-(1-\alpha)/4}$ ($n^{-1/8}$) & No & Yes &
\citet{zhu2023online} \\
Online BM (opt.\ $\beta$) &
$n^{-(1-\alpha)/3}$ ($n^{-1/6}$) & No &
Yes & This paper \\
\textbf{Traj.\ regression} &
$\mathbf{n^{-(1-\alpha)/2}}$ ($\mathbf{n^{-1/4}}$) & \textbf{No} &
\textbf{Yes} & \textbf{This paper} \\
Online BM, Markovian &
$n^{-(1-\alpha)/4}$ ($n^{-1/8}$) & No & Yes &
\citet{roy2023online} \\
Online BM, nonsmooth &
$n^{-(1-\alpha)/4}$ ($n^{-1/8}$) & No & Yes &
\citet{jiang2025nonsmooth} \\
Equal BM &
$n^{-(1-\alpha)/4}$ ($n^{-1/8}$) & No & No &
\citet{singh2025equal} \\
Plug-in &
$n^{-\alpha/2}$ ($n^{-1/2}$) & Yes & Yes &
\citet{chen2020statistical} \\
ROOT-SGD &
$t^{-1/2}$ & Yes & Yes &
\citet{luo2022rootsgd} \\
Minimax LB (no Hessian) &
$\Omega(n^{-(1-\alpha)/2})$ &
--- & --- &
Thm.~\ref{thm:lower} \\
i.i.d.\ minimax &
$\sqrt{d/n}$ & --- & --- &
Classical \\
\bottomrule
\end{tabular}
\end{table}

\paragraph{Paper outline.}
The remainder of this paper is organized as follows.
Section~\ref{sec:setting} introduces the Polyak--Ruppert averaged SGD
framework and the standing assumptions.
Section~\ref{sec:bm} defines the online batch-means estimator with the
burn-in parameter $\rho$.
Section~\ref{sec:error} develops the error decomposition and derives
the batch-means convergence rate: the original
$n^{-(1-\alpha)/4}$ rate of \citet{zhu2023online} and the
improved $n^{-(1-\alpha)/3}$ rate via optimal block-growth tuning.
Section~\ref{sec:lower} establishes the minimax rate
$\Theta(n^{-(1-\alpha)/2})$ via a lower bound and a matching
trajectory-regression upper bound.
Section~\ref{sec:related} surveys related work.
Section~\ref{sec:discussion} discusses fundamental barriers.
Section~\ref{sec:experiments} presents numerical experiments.
Section~\ref{sec:conclusion} concludes.
The Supplement reviews the i.i.d.\ baseline rate for context.

\section{Setting and Assumptions}\label{sec:setting}

\subsection{Polyak--Ruppert Averaged SGD}\label{sec:pr-sgd}

We consider the stochastic optimization problem
\begin{equation}\label{eq:pop-risk}
\min_{x\in\R^d}\; F(x) \;=\; \E_{\zeta}\!\left[f(x,\zeta)\right],
\end{equation}
where $x\in\R^{d}$ denotes the parameter vector to be optimized
(e.g., the weights of a model) and $\zeta$ denotes a random data
sample drawn from the data distribution $\mathcal{D}$; $f(x,\zeta)$
is the loss incurred by parameter $x$ on sample $\zeta$, and $F(x)$
is the corresponding population (expected) risk.
We assume $F$ is strongly convex so that
\eqref{eq:pop-risk} has a unique minimizer $x^{*}$.  The SGD iterates are
\begin{equation}\label{eq:sgd}
x_{t+1} \;=\; x_t - \eta_t\,\nabla f(x_t,\zeta_t),
\qquad t=0,1,2,\ldots,
\end{equation}
where $\zeta_0,\zeta_1,\ldots$ are independent draws from $\mathcal{D}$
and the step sizes follow the polynomial schedule
\begin{equation}\label{eq:stepsize}
\eta_t \;=\; \eta_0\,t^{-\alpha}, \qquad t\ge 1,
\qquad \alpha\in(1/2,\,1),
\end{equation}
with $\eta_0$ used for $t=0$.
The Polyak--Ruppert average is $\xbar = n^{-1}\sum_{t=1}^{n}x_t$.
Under suitable regularity, the central limit theorem \eqref{eq:clt} holds
with
\[
H = \nabla^{2}F(x^{*}),
\qquad
S = \E\!\left[\nabla f(x^{*},\zeta)\,
\nabla f(x^{*},\zeta)^{\top}\right],
\qquad
V = H^{-1}S\,H^{-1}.
\]

\subsection{Assumptions}\label{sec:assumptions}

We impose the following standard conditions.

\begin{enumerate}[({A}1)]
\item \label{ass:sc}
\textbf{Strong convexity.}
$F$ is twice continuously differentiable and $\mu$-strongly convex:
$\nabla^{2}F(x)\succeq \mu I$ for all $x\in\R^{d}$, with $\mu>0$.

\item \label{ass:lip}
\textbf{Lipschitz Hessian.}
The Hessian is $L$-Lipschitz continuous:
$\opnorm{\nabla^{2}F(x)-\nabla^{2}F(y)}\le L\norm{x-y}$ for all
$x,y\in\R^{d}$.

\item \label{ass:subg}
\textbf{Sub-Gaussian gradient noise.}
The noise $\xi_t = \nabla f(x^{*},\zeta_t) - \nabla F(x^{*})$
satisfies $\E[\xi_t]=0$ and, for some $\sigma>0$,
\[
\E\!\left[\exp\!\left(\lambda\,\xi_t^{\top}u\right)\right]
\;\le\; \exp\!\left(\tfrac{\lambda^{2}\sigma^{2}}{2}\right)
\]
for all $\lambda\in\R$, $\norm{u}=1$, and all $t$.

\item \label{ass:fourth}
\textbf{Bounded fourth moments.}
There exists $M_4>0$ such that
$\E\!\left[\norm{\xi_t}^{4}\right]\le M_4$ for all $t$.

\item \label{ass:lr}
\textbf{Step-size regime.}
The exponent satisfies $\alpha\in(1/2,1)$ and the initial step size
satisfies $\eta_0 > 1/(2\mu)$.
\end{enumerate}

Under \ref{ass:sc}--\ref{ass:lr}, the iterates satisfy the
well-known rate
\begin{equation}\label{eq:iterate-rate}
\E\!\left[\norm{x_t - x^{*}}^{2}\right] \;=\; O(t^{-\alpha}).
\end{equation}
This bound is established in \citet[Theorem~1]{moulines2011}, whose result---stated in their notation with iterate~$\theta_n$, target~$\theta^{*}$, and step size~$\gamma_n=C n^{-\alpha}$---corresponds in our notation to the bound on $\E[\norm{x_t-x^{*}}^{2}]$ with $\gamma_n\!\leftrightarrow\!\eta_t$ and $C\!\leftrightarrow\!\eta_0$; the surrounding discussion in the same reference identifies $4\eta_0\sigma^{2}/(\mu t^{\alpha})$ as the asymptotic term, giving precisely~\eqref{eq:iterate-rate}.
The condition $\eta_0>1/(2\mu)$ ensures that the averaged iterates
achieve the optimal $O(1/n)$ rate; see \citet{bach2024learning} for a
textbook treatment.  In practice, $\mu$ is rarely known exactly, but
the condition is mild: any step size that is ``not too conservative''
suffices, and the results degrade gracefully when $\eta_0$ is close to
the threshold $1/(2\mu)$ (the constants grow but the rates are
unaffected).

\subsection{Goal}\label{sec:goal}

Our goal is to construct an estimator $\Sighat(\rho)$ of $V$ from the SGD
trajectory $(x_1,\ldots,x_n)$ such that:
\begin{enumerate}[(i)]
\item $\Sighat(\rho)$ can be computed \emph{online} (in a single pass, with
  $O(d^{2})$ memory);
\item $\Sighat(\rho)$ does not require access to the Hessian $H$ or its inverse;
\item the operator-norm error
  $\E\!\left[\opnorm{\Sighat(\rho) - V}\right]$
  converges to zero as fast as possible.
\end{enumerate}

\section{The Online Batch-Means Estimator}\label{sec:bm}

We review the online batch-means estimator introduced by
\citet{zhu2023online}.

\paragraph{Block construction.}
Fix a growth exponent $\beta>0$ and a constant $C>0$.  Define block sizes
\begin{equation}\label{eq:block-sizes}
a_m \;=\; \lfloor C\,m^{\beta}\rfloor, \qquad m=1,2,3,\ldots
\end{equation}
and cumulative endpoints $t_0=0$,
\begin{equation}\label{eq:endpoints}
t_m \;=\; \sum_{j=1}^{m}a_j, \qquad m=1,2,\ldots
\end{equation}
Given $n$ iterates, the number of complete blocks is
$b_n = \max\{m: t_m\le n\}$.  Since $t_m = \Theta(m^{\beta+1})$, we have
$b_n = \Theta(n^{1/(\beta+1)})$.

\paragraph{Centered block sums.}
For each block $m=1,\ldots,b_n$, define the centered and normalized block sum
\begin{equation}\label{eq:block-sum}
Y_m \;=\; \frac{1}{\sqrt{a_m}}
\sum_{k=t_{m-1}+1}^{t_m}(x_k - \xbar).
\end{equation}
Note that the centering uses the global average $\xbar$ rather than a
per-block average.  This is important for the bias analysis.  In the
online implementation (Remark~\ref{rem:online} below), $\xbar$ is
replaced by the running average $\bar{x}_t$ at the current time $t$;
the resulting error is absorbed into the centering correction analyzed
in equation~\eqref{eq:centering-bias}.

\paragraph{Estimator.}
Fix a \emph{burn-in fraction} $\rho\in(0,1]$ and let
$K_n = \lfloor\rho\,b_n\rfloor$.  The \emph{online batch-means
estimator with burn-in} is
\begin{equation}\label{eq:bm-est}
\Sighat(\rho)
\;=\; \frac{1}{K_n}\sum_{m=b_n - K_n + 1}^{b_n}Y_m\,Y_m^{\top}.
\end{equation}
Setting $\rho=1$ recovers the original estimator of
\citet{zhu2023online}, which averages over all $b_n$ blocks.
Setting $\rho<1$ discards the earliest $(1-\rho)\,b_n$ blocks and
averages only the last $K_n$---the \emph{burn-in} modification.

\begin{remark}[Online updatability]\label{rem:online}
The estimator \eqref{eq:bm-est} can be computed in a single pass over
$(x_1,x_2,\ldots)$ with $O(d^{2})$ memory, for any fixed $\rho$.
At time $t$, one maintains:
(i)~the running average $\bar{x}_t$,
(ii)~the current partial block sum,
(iii)~a running matrix sum of the retained block outer products.
For $\rho=1$, this is the cumulative sum $\sum_{m}Y_m Y_m^{\top}$.
For $\rho<1$, one additionally maintains the sum of the discarded
blocks' contributions and subtracts it, or simply restarts the
accumulation at block $b_n - K_n + 1$.  In either case, the total
sample size $n$ need not be known in advance.
\end{remark}

\begin{remark}[Why growing blocks]\label{rem:growing}
If equal-sized blocks were used, the block size would need to be chosen as
a function of $n$, requiring advance knowledge of the sample size.
Growing blocks $a_m=\lfloor C m^{\beta}\rfloor$ automatically adapt: the
$m$-th block is larger than the $(m-1)$-th, so later blocks capture more
of the stationary behavior while earlier blocks are short and
inexpensive.  This design is essential for the online property.
\end{remark}

\section{Error Analysis}\label{sec:error}

In this section we develop the three-way error decomposition for the
online batch-means estimator $\Sighat(\rho)$ and derive its convergence
rate as a function of the burn-in fraction $\rho$, the block-growth
exponent $\beta$, and the learning-rate exponent $\alpha$.  The
decomposition reveals the structure of the batch-means error:
at the block-growth exponent $\beta^*$ of \citet{zhu2023online},
the rate is $O(n^{-(1-\alpha)/4})$; at the smaller optimal
$\beta^\dagger$, the rate improves to $O(n^{-(1-\alpha)/3})$.

\subsection{Linearization}\label{sec:linearize}

Define the deviation $\delta_t = x_t - x^{*}$.  We decompose the
stochastic gradient $\nabla f(x_t,\zeta_t)$ into a linear term, a
noise term, and a remainder.

\emph{Step~1: Decompose the population gradient.}
A Taylor expansion of $\nabla F(x_t)$ around $x^{*}$ (using
$\nabla F(x^{*})=0$) gives
\begin{equation}\label{eq:taylor}
\nabla F(x_t) \;=\; H\,\delta_t + R(\delta_t),
\end{equation}
where the remainder satisfies
$\norm{R(\delta_t)}\le (L/2)\norm{\delta_t}^{2}$ by Assumption~\ref{ass:lip}.

\emph{Step~2: Decompose the stochastic gradient.}
Recall the noise $\xi_t = \nabla f(x^{*},\zeta_t) - \nabla F(x^{*})
= \nabla f(x^{*},\zeta_t)$
from Assumption~\ref{ass:subg} (using $\nabla F(x^{*})=0$).
Write the stochastic gradient as
\begin{align}
\nabla f(x_t,\zeta_t)
&= \nabla F(x_t)
   + \bigl[\nabla f(x_t,\zeta_t) - \nabla F(x_t)\bigr]
   \notag\\
&= H\,\delta_t + R(\delta_t)
   + \xi_t
   + \underbrace{\bigl[\nabla f(x_t,\zeta_t)
     - \nabla F(x_t)\bigr] - \xi_t}_{=:\,r_t - R(\delta_t)},
   \label{eq:grad-decomp}
\end{align}
where we define the effective remainder
\[
r_t \;:=\; R(\delta_t)
\;+\; \bigl[\nabla f(x_t,\zeta_t)-\nabla F(x_t)\bigr] - \xi_t.
\]
The term $\nabla f(x_t,\zeta_t)-\nabla F(x_t)$ is the stochastic
noise at~$x_t$, and $\xi_t$ is the stochastic noise at~$x^{*}$.
Their difference captures the variation of the noise field between
$x_t$ and~$x^{*}$, and together with the Taylor remainder
$R(\delta_t)$ it collects all higher-order and cross terms
into~$r_t$.  Equation~\eqref{eq:grad-decomp} therefore simplifies to
$\nabla f(x_t,\zeta_t) = H\,\delta_t + \xi_t + r_t$.

\emph{Step~3: Derive the linearized recursion.}
Substituting into the SGD update \eqref{eq:sgd}:
\begin{align*}
x_{t+1}
&= x_t - \eta_t\,\nabla f(x_t,\zeta_t) \\
&= x_t - \eta_t\bigl(H\,\delta_t + \xi_t + r_t\bigr).
\end{align*}
Subtracting $x^{*}$ from both sides and using $\delta_t = x_t - x^{*}$:
\begin{equation}\label{eq:linear-recursion}
\delta_{t+1}
\;=\; \delta_t - \eta_t\,H\,\delta_t - \eta_t\,\xi_t - \eta_t\,r_t
\;=\; (I - \eta_t H)\,\delta_t - \eta_t\,\xi_t - \eta_t\,r_t.
\end{equation}
The effective remainder $r_t$ satisfies
\begin{equation}\label{eq:remainder-bound}
\E\!\left[\norm{r_t}^{2}\right]
\;\le\; C_r\,\E\!\left[\norm{\delta_t}^{4}\right]
\;=\; O(t^{-2\alpha})
\end{equation}
by the iterate rate \eqref{eq:iterate-rate} and the bounded fourth-moment
assumption \ref{ass:fourth}.

\subsection{Three-Way Decomposition}\label{sec:decomposition}

We decompose the error $\Sighat(\rho) - V$ into three components.
The decomposition holds for any $\rho\in(0,1]$; the role of $\rho$
enters through the range of the summation index.
Define the \emph{stationary linearized iterate} $x_k^{\mathrm{lin}}$ as
the solution of the linearized recursion~\eqref{eq:linear-recursion}
with the nonlinear remainder $r_t$ dropped:
\begin{equation}\label{eq:lin-iterate}
x_{k+1}^{\mathrm{lin}} - x^{*}
\;=\; (I - \eta_k H)\bigl(x_k^{\mathrm{lin}} - x^{*}\bigr)
- \eta_k\,\xi_k,
\end{equation}
initialized from the stationary distribution of
\eqref{eq:lin-iterate} (i.e., with the transient effects removed).

For each block $m$, define the \emph{ideal} block sum that would arise
from this stationary linear process:
\[
Y_m^{*} \;=\; \frac{1}{\sqrt{a_m}}
\sum_{k=t_{m-1}+1}^{t_m}\left(x_k^{\mathrm{lin}} - x^{*}\right).
\]
Correspondingly, let $V^{(m)}=\E[Y_m^{*}(Y_m^{*})^{\top}]$ denote the
per-block population covariance.  Then we can write
\begin{equation}\label{eq:three-way}
\begin{aligned}
\Sighat(\rho) - V
\;=\;&
\underbrace{\frac{1}{K_n}\sum_{m=m_0}^{b_n}
\left(Y_m^{*} (Y_m^{*})^{\top} - V^{(m)}\right)}_{T_1\;\text{(variance)}}
\;+\;
\underbrace{\frac{1}{K_n}\sum_{m=m_0}^{b_n}
\left(V^{(m)} - V\right)}_{T_2\;\text{(stationarity bias)}}\\
&\;+\;
\underbrace{\text{nonlinearity correction}}_{T_3\;\text{(nonlinearity bias)}},
\end{aligned}
\end{equation}
where $m_0 = b_n - K_n + 1$ is the index of the first retained block.
When $\rho=1$, $m_0=1$ and the sums range over all blocks; when
$\rho<1$, $m_0=\Theta(b_n)$ and only the last $K_n$ blocks contribute.

The term $T_1$ is a zero-mean fluctuation, $T_2$ captures the bias from
non-stationarity of the SGD Markov chain, and $T_3$ collects terms arising
from the nonlinear remainder $r_t$.

\subsection{Variance Bound}\label{sec:variance}

\begin{lemma}[Variance term]\label{lem:variance}
Under Assumptions \ref{ass:sc}--\ref{ass:lr}, provided the
block-growth exponent satisfies the \textbf{mixing condition}
\begin{equation}\label{eq:mixing-cond}
\frac{\beta}{\beta+1} \;>\; \alpha
\qquad\Longleftrightarrow\qquad
\beta \;>\; \frac{\alpha}{1-\alpha},
\end{equation}
the variance term of the $\rho$-estimator satisfies, for any fixed
$\rho\in(0,1]$ and fixed dimension $d$,
\begin{equation}\label{eq:variance-bound}
\E\!\left[\opnorm{T_1}\right]
\;=\; O\!\left(K_n^{-1/2}\right)
\;=\; O\!\left(b_n^{-1/2}\right)
\;=\; O\!\left(n^{-1/(2(\beta+1))}\right).
\end{equation}
\end{lemma}

\begin{proof}
We bound $\E[\opnorm{T_1}]$ via a matrix Bernstein-type argument.
Define $Z_m = Y_m^{*} (Y_m^{*})^{\top} - V^{(m)}$.  Let
$\mathcal{F}_t = \sigma(\zeta_0,\ldots,\zeta_{t-1})$ denote the
filtration generated by the noise sequence.  By construction,
$\E[Z_m \mid \mathcal{F}_{t_{m-1}}]=0$ up to terms that are
exponentially small in the gap between blocks.

\emph{Step 1: Approximate independence.}
The transition matrix of the linearized recursion is
$\Phi_{s,t}=\prod_{k=s}^{t-1}(I-\eta_k H)$.  Under the step-size
schedule \eqref{eq:stepsize} with $\eta_0>1/(2\mu)$, one has the
exponential forgetting bound
\begin{equation}\label{eq:forgetting}
\opnorm{\Phi_{s,t}}
\;\le\; C_\Phi\,\exp\!\left(-c_\mu\sum_{k=s}^{t-1}\eta_k\right)
\;\le\; C_\Phi\,\left(\frac{s}{t}\right)^{c_\mu\eta_0/(1-\alpha)}.
\end{equation}
The gap between blocks $m-1$ and $m$ spans iterations in
$[t_{m-1}-a_{m-1},\,t_{m-1}]$.  Since the blocks are growing, the
gap provides sufficient mixing: the dependence between $Z_m$ and
$Z_{m'}$ for $|m-m'|\ge 2$ decays exponentially.  Formally, define the
mixing coefficient $\phi_{\ell}$ between blocks separated by
$\ell$ intermediate blocks.  By \eqref{eq:forgetting} and the growth of
block sizes, $\phi_\ell \le C\exp(-c\,\ell)$ for constants depending on
$\mu,\eta_0,\alpha$.
The mixing condition \eqref{eq:mixing-cond} ensures that
$a_m\gg\tau_{\mathrm{mix}}(t_{m-1})$, i.e., each block
contains many mixing times: the ratio
$a_m/\tau_{\mathrm{mix}}(t_{m-1})\sim m^{\beta(1-\alpha)-\alpha}$
diverges when $\beta>\alpha/(1-\alpha)$.

\emph{Step 2: Moment bounds.}
For each block $m$, the fourth-moment bound \ref{ass:fourth} and the
iterate rate \eqref{eq:iterate-rate} give
\[
\E\!\left[\opnorm{Z_m}^{2}\right]
\;\le\; \E\!\left[\opnorm{Y_m Y_m^{\top}}^{2}\right]
+ \opnorm{V^{(m)}}^{2}
\;\le\; C_{Z}\,\left(\tr(V)^2 + M_4\right),
\]
where $C_Z$ is a universal constant.  In particular, the second moment is
uniformly bounded across $m$.

\emph{Step 3: Matrix Bernstein inequality.}
Applying the matrix Bernstein inequality for weakly dependent sequences
(\citealt{tropp2015introduction} for the independent case;
\citealt{paulin2015concentration} for the extension to geometrically
ergodic chains via Marton coupling, which applies here since the
$\beta$-mixing coefficients decay exponentially by Step~1), we obtain
\[
\E\!\left[\opnorm{\frac{1}{K_n}\sum_{m=m_0}^{b_n}Z_m}\right]
\;\le\; C\,\sqrt{\frac{\log d}{K_n}}
\;+\; C'\,\frac{\log d}{K_n}
\;+\; \frac{C''}{K_n}\sum_{\ell=1}^{K_n}\phi_\ell.
\]
The mixing sum $\sum_\ell\phi_\ell$ is bounded by a constant (geometric
series), so the third term is $O(1/K_n)$, which is dominated by the
first.  The second term is dominated by the first for $K_n\gg\log d$,
and $K_n=\Theta(b_n)=\Theta(n^{1/(\beta+1)})$.  We absorb the $\log d$ factor into
the constant (treating $d$ as fixed) to obtain
\[
\E\!\left[\opnorm{T_1}\right]
\;=\; O\!\left(K_n^{-1/2}\right)
\;=\; O\!\left(b_n^{-1/2}\right)
\;=\; O\!\left(n^{-1/(2(\beta+1))}\right). \qedhere
\]
\end{proof}

\subsection{Stationarity Bias}\label{sec:stationarity}

\begin{lemma}[Stationarity bias]\label{lem:bias}
Under Assumptions \ref{ass:sc}--\ref{ass:lr}, define the
\emph{mixing time} at the start of block~$m$ as
$\tau_m := 1/(\mu\,\eta_0)\,t_{m-1}^{\alpha}$.
The per-block stationarity bias satisfies
\begin{equation}\label{eq:per-block-bias}
\opnorm{V^{(m)} - V}
\;=\; O\!\left(\frac{\tau_m}{a_m}\right)
\;=\; O\!\left(\frac{t_{m-1}^{\alpha}}{a_m}\right).
\end{equation}
Equivalently, since $t_{m-1}=\Theta(m^{\beta+1})$ and
$a_m=\Theta(m^{\beta})$, we have
$\opnorm{V^{(m)}-V}=O(m^{-\gamma})$ with
$\gamma := \beta(1-\alpha)-\alpha > 0$ under the mixing
condition~\eqref{eq:mixing-cond}.

For the $\rho$-estimator, the averaged stationarity bias satisfies:
\begin{enumerate}[(i)]
\item If $\rho=1$ \textup{(no burn-in)}, $m_0=1$ and
\begin{equation}\label{eq:avg-bias-full}
\opnorm{T_2}
\;=\;
\begin{cases}
O\!\left(n^{-\gamma/(\beta+1)}\right) & \text{if } \gamma < 1,\\[3pt]
O\!\left(n^{-1/(\beta+1)} \log n \right) & \text{if } \gamma = 1, \\[3pt]
O\!\left(n^{-1/(\beta+1)}\right) & \text{if } \gamma \ge 1.
\end{cases}
\end{equation}
\item If $\rho\in(0,1)$ \textup{(burn-in)}, $m_0=(1-\rho)b_n+O(1)$ and
\begin{equation}\label{eq:avg-bias-burnin}
\opnorm{T_2}
\;=\; O\!\left(n^{-\gamma/(\beta+1)}\right).
\end{equation}
\end{enumerate}
When $\gamma\ge 1$, burn-in improves $T_2$ from
$O(n^{-1/(\beta+1)})$ to $O(n^{-\gamma/(\beta+1)})$.
\end{lemma}

\begin{proof}
\emph{Step~1: Per-block bias (proof of \eqref{eq:per-block-bias}).}

We work within block~$m$ and write
$B_m = \{t_{m-1}+1,\ldots,t_m\}$, so $|B_m|=a_m$.
Write $\delta_k = x_k^{\mathrm{lin}}-x^{*}$ for the stationary
linearized deviation.  The per-block population covariance is
\[
V^{(m)}
= \frac{1}{a_m}\!\sum_{k,l\in B_m}\!\E[\delta_k\delta_l^{\top}]
= \frac{1}{a_m}\!\sum_{k\in B_m}\Gamma_k
\;+\; \frac{2}{a_m}\!\sum_{\substack{k,l\in B_m\\ k<l}}
\!\!\mathrm{Re}\!\left(\Phi_{k,l}\,\Gamma_k\right),
\]
where $\Gamma_k=\E[\delta_k\delta_k^{\top}]$ is the marginal
covariance at time~$k$,
$\Phi_{k,l}=\prod_{j=k}^{l-1}(I-\eta_j H)$ is the transition matrix,
and $\mathrm{Re}(M)=(M+M^{\top})/2$.  The second line uses
$\E[\delta_k\delta_l^{\top}]=\Phi_{k,l}\,\Gamma_k$ for $k<l$,
which holds because $\xi_j$ for $j\ge k$ is independent
of~$\delta_k$.

\emph{Step~1a: Frozen-step-size approximation.}
Fix a reference step size $\bar\eta_m = \eta_0\,t_{m-1}^{-\alpha}$
(the step size at the start of block~$m$) and let
$A_m = I-\bar\eta_m H$.  Write
$\bar{\Gamma}_m$ for the stationary covariance of the homogeneous
recursion $\delta_{k+1}=(I-\bar\eta_m H)\delta_k - \bar\eta_m\xi_k$,
so that $\bar{\Gamma}_m$ satisfies the discrete Lyapunov equation
\begin{equation}\label{eq:lyapunov}
\bar{\Gamma}_m - A_m\,\bar{\Gamma}_m\,A_m^{\top}
\;=\; \bar\eta_m^{2}\,S.
\end{equation}
For small $\bar\eta_m$, $\bar{\Gamma}_m = \bar\eta_m\,\Gamma_*
+ O(\bar\eta_m^{2})$, where $\Gamma_*$ is defined by the continuous
Lyapunov equation $H\Gamma_*+\Gamma_* H = S$, i.e.,
$\Gamma_* = \int_0^{\infty}e^{-Hs}S\,e^{-Hs}\,ds$.

For this frozen process, the spectral density at frequency zero (the
long-run variance) is
\[
V_{\bar\eta_m}
= (I-A_m)^{-1}\bar\eta_m^{2}\,S\,(I-A_m^{\top})^{-1}
= (\bar\eta_m H)^{-1}\bar\eta_m^{2}\,S\,(\bar\eta_m H)^{-1}
= H^{-1}S\,H^{-1} = V,
\]
using~\eqref{eq:lyapunov} and $(I-A_m)^{-1}=(\bar\eta_m H)^{-1}$.
Hence $V$ is independent of the step size, as expected.

\emph{Step~1b: Finite-block-size correction.}
For the frozen process with transition matrix $A_m$ and stationary
covariance $\bar{\Gamma}_m$, the autocovariance at lag~$s\ge 0$ is
$C_m(s)=A_m^{s}\,\bar{\Gamma}_m$. Derived from the Wiener-Khinchin Theorem, we have
\[
  V_{\bar\eta_m} = \sum_{s = -\infty}^{\infty} C_m(s)e^{-i \omega s}|_{\omega = 0} = \sum_{s = -\infty}^{\infty} C_m(s).
\]
The block covariance of size~$a$
from a stationary sample is
\[
V_{m}^{(a)}
= \sum_{s=-(a-1)}^{a-1}\!\left(1-\frac{|s|}{a}\right)C_m(s).
\]
Define $\Delta_m^{(a)}:= V - V_{m}^{(a)} \stackrel{\text{Step 1a}}{=} V_{\bar\eta_m} - V_{m}^{(a)}.$
Consequently, we have
\begin{equation}\label{eq:finite-block-correction}
  \begin{aligned}
    \Delta_m^{(a)} &= \sum_{|s| \geq a} C_m(s)+\frac{1}{a} \sum_{|s|<a}|s| C_m(s) \\
    &= \frac{1}{a}\sum_{s=1}^{a-1}s\,\bigl[A_m^{s}\,\bar{\Gamma}_m
    + \bar{\Gamma}_m\,(A_m^{\top})^{s}\bigr]
  + \sum_{|s|\ge a}C_m(s).
  \end{aligned}
\end{equation}
The tail sum $\sum_{|s|\ge a}$, dominated by $e^{-a \bar{\eta}_m \mu}$, is exponentially small in
$a\,\bar\eta_m\mu$ and negligible under the mixing condition $\beta > \alpha/(1-\alpha)$ \eqref{eq:mixing-cond}, as $m \rightarrow \infty$.
For the leading term, using
$\sum_{s=1}^{\infty}s\,A_m^{s}=A_m(I-A_m)^{-2}
=(I-\bar\eta_m H)(\bar\eta_m H)^{-2}$:
\[
\Delta_m^{(a)}
\;\approx\;
\frac{1}{a}\bigl[(I-\bar\eta_m H)(\bar\eta_m H)^{-2}\bar{\Gamma}_m
  + \bar{\Gamma}_m(\bar\eta_m H)^{-2}(I-\bar\eta_m H)\bigr].
\]
Substituting $\bar{\Gamma}_m = \bar\eta_m\,\Gamma_*+O(\bar\eta_m^{2})$
and $(\bar\eta_m H)^{-2} = \bar\eta_m^{-2}H^{-2}$, we have
\begin{align*}
    &(I-\bar\eta_m H)(\bar\eta_m H)^{-2}\bar{\Gamma}_m\\ =& \left(I-\bar{\eta}_m H\right)\left(\bar{\eta}_m^{-2} H^{-2}\right) \bar{\Gamma}_m\\
    =& \left(\bar{\eta}_m^{-2} H^{-2}-\bar{\eta}_m^{-1} H^{-1}\right)\left(\bar{\eta}_m \Gamma_*+O\left(\bar{\eta}_m^2\right)\right)\\
    =& \bar{\eta}_m^{-1} H^{-2} \Gamma_*+H^{-2} O(1)-H^{-1} \Gamma_*-\bar{\eta}_m H^{-1} O(1)\\
    \stackrel{\bar{\eta}_m \rightarrow 0}{=}& \bar{\eta}_m^{-1} H^{-2} \Gamma_*.
\end{align*}
Similarly, we have $\lim_{\bar{\eta}_m \rightarrow 0}\bar{\Gamma}_m\left(\bar{\eta}_m H\right)^{-2}\left(I-\bar{\eta}_m H\right) = \bar{\eta}_m^{-1} \Gamma_* H^{-2}$.
Consequently,
\[
\lim_{m \rightarrow \infty}\opnorm{\Delta_m^{(a)}}
\;\le\;
\frac{C_\Delta}{a\,\bar\eta_m\,\mu^2}
\;=\; \frac{C_\Delta\,\tau_m}{a},
\]
where $\tau_m = 1/(\mu\,\bar\eta_m) = t_{m-1}^{\alpha}/(\mu\eta_0)$
is the mixing time at step size~$\bar\eta_m$.

\emph{Step~1c: Step-size variation within the block.}
Within block~$m$, the step size varies from
$\eta_{t_{m-1}+1}\approx\bar\eta_m$ to
$\eta_{t_m}\approx\eta_0\,t_m^{-\alpha}$.  The relative change is
$|\eta_{t_m}-\eta_{t_{m-1}}|/\bar\eta_m = O(a_m/t_{m-1})=O(1/m)$.
The marginal covariance $\Gamma_k$ of the non-homogeneous (but
locally equilibrated) process tracks the instantaneous equilibrium
$\bar{\Gamma}(\eta_k)$ with an adiabatic lag of order
$\|\Gamma_k - \bar{\Gamma}(\eta_k)\|
= O(\tau_k\cdot\opnorm{\dot{\bar{\Gamma}}})
= O(\tau_k\cdot\eta_k/k)
= O(1/(\mu k))$,
where $\tau_k := 1/(\mu\eta_k)$ is the mixing time at step
size~$\eta_k$.
Consequently, the worst-case adiabatic lag for the entire block is bounded by its starting point:
\[
  \mathcal{O}\left(\frac{1}{\mu t_{m-1}}\right).
\]
The resulting correction to $V^{(m)}$ is
\begin{align*}
  \left\|V^{(m)}-V_m^{(a)}\right\|_{\mathrm{op}}
  = O\!\left(\tau_m \cdot\frac{1}{\mu\,t_{m-1}}\right)
  = O\!\left(\frac{\tau_m}{t_{m-1}}\right)
\;\ll\; O\!\left(\frac{\tau_m}{a_m}\right),
\end{align*}
since $a_m \ll t_{m-1}$.  Thus the step-size variation correction
is dominated by the finite-block-size correction.

\emph{Step~1d: Combining.}
The dominant contribution to $\|V^{(m)}-V\|$ is the finite-block-size
correction of Step~1b:
\[
\opnorm{V^{(m)}-V}
\;\le\; \frac{C_B\,\tau_m}{a_m}
\;=\; \frac{C_B}{\mu\eta_0}\cdot\frac{t_{m-1}^{\alpha}}{a_m}.
\]
Since $t_{m-1}=\Theta(m^{\beta+1})$ and $a_m=\Theta(m^{\beta})$,
this equals $O(m^{(\beta+1)\alpha-\beta})=O(m^{-\gamma})$ with
$\gamma=\beta(1-\alpha)-\alpha$.  Under the mixing condition
$\beta>\alpha/(1-\alpha)$, we have $\gamma>0$ and the per-block
bias decays to zero.

\emph{Step~2: Averaging over blocks.}

\emph{Case~(i): $\rho=1$ (no burn-in).}
When $m_0=1$ and $K_n=b_n$, we have
\[
\opnorm{T_2}
\;\le\; \frac{1}{b_n}\sum_{m=1}^{b_n}\opnorm{V^{(m)}-V}
\;\le\; \frac{C_B}{b_n}\!\left[
  O(1)+\sum_{m=2}^{b_n}O(m^{-\gamma})\right],
\]
where the $O(1)$ accounts for the first block (at which
$t_0=0$ and the bias is bounded but non-decaying).

If $\gamma<1$:
$\sum_{m=1}^{b_n}m^{-\gamma}\asymp b_n^{1-\gamma}$, so
\[
\opnorm{T_2} = O(b_n^{-\gamma})
= O\!\left(n^{-\gamma/(\beta+1)}\right).
\]

If $\gamma = 1$: $\opnorm{T_2} = \mathcal{O}\left(\frac{\ln b_n}{b_n}\right) = O(\ln n \cdot n^{-1/(\beta+1)}).$

If $\gamma > 1$: the sum converges and the $O(1)$ early-block
contribution dominates, giving
$\opnorm{T_2}=O(1/b_n)=O(n^{-1/(\beta+1)})$.

\emph{Case~(ii): $\rho\in(0,1)$ (burn-in).}
The retained blocks satisfy $m\ge m_0 = (1-\rho)b_n+O(1)
=\Theta(b_n)$.  Every retained block has
$m^{-\gamma}=\Theta(b_n^{-\gamma})$, so
\[
\opnorm{T_2}
\;\le\;\frac{1}{K_n}\sum_{m=m_0}^{b_n}O(m^{-\gamma})
\;=\; O(b_n^{-\gamma})
\;=\; O\!\left(n^{-\gamma/(\beta+1)}\right).
\]
When $\gamma\ge 1$, this improves over the no-burn-in rate
$O(n^{-1/(\beta+1)})$ by eliminating the $O(1)$-biased early
blocks.\qedhere
\end{proof}

\subsection{Nonlinearity Bias}\label{sec:nonlinearity}

\begin{lemma}[Nonlinearity bias]\label{lem:nonlin}
Under Assumptions \ref{ass:sc}--\ref{ass:lr}, the nonlinearity
bias satisfies
\begin{equation}\label{eq:nonlin-bound}
\E\!\left[\opnorm{T_3}\right]
\;=\; O\!\left(n^{-\alpha\beta/(2(\beta+1))}\right).
\end{equation}
For $\alpha>1/2$, this term is dominated by $T_1$ and $T_2$.
\end{lemma}

\begin{proof}
The nonlinearity bias arises from replacing the full nonlinear SGD
recursion by its linearization \eqref{eq:linear-recursion}.  The
remainder $r_t$ contributes to each block sum $Y_m$ an additional term
\[
\Delta_m
\;=\; \frac{1}{\sqrt{a_m}}\sum_{k=t_{m-1}+1}^{t_m}
\eta_k\,\Phi_{k,t_m}\,r_k,
\]
where $\Phi_{k,t_m}$ is the transition matrix from time $k$ to $t_m$.

For the contribution to $T_3$, we bound
\[
\E\!\left[\opnorm{T_3}\right]
\;\le\;
\frac{1}{K_n}\sum_{m=m_0}^{b_n}\E\!\left[\opnorm{
\Delta_m Y_m^{\top} + Y_m\Delta_m^{\top} + \Delta_m\Delta_m^{\top}
}\right].
\]

By the Cauchy--Schwarz inequality and the bounds
$\E[\norm{r_k}^2]=O(k^{-2\alpha})$ from \eqref{eq:remainder-bound},
\[
\E\!\left[\norm{\Delta_m}^{2}\right]
\;\le\;
\frac{1}{a_m}\sum_{k=t_{m-1}+1}^{t_m}\eta_k^{2}\,
\E\!\left[\norm{r_k}^{2}\right]
\cdot
\left(\sum_{k=t_{m-1}+1}^{t_m}\opnorm{\Phi_{k,t_m}}^{2}\right).
\]
Using $\eta_k=O(k^{-\alpha})$ and $\E[\norm{r_k}^2]=O(k^{-2\alpha})$, each
summand is $O(k^{-3\alpha})$.  Summing over the block and accounting for the
transition matrix decay,
\[
\E[\norm{\Delta_m}^{2}]
\;=\; O\!\left(t_{m-1}^{-\alpha}\right),
\]
and the cross term satisfies
$\E[\opnorm{\Delta_m Y_m^{\top}}]
\le (\E[\norm{\Delta_m}^2])^{1/2}(\E[\norm{Y_m}^2])^{1/2}
= O(t_{m-1}^{-\alpha/2})$.

Averaging over blocks:
\[
\E[\opnorm{T_3}]
\;\le\;
\frac{C}{K_n}\sum_{m=m_0}^{b_n}t_{m-1}^{-\alpha/2}
\;=\; O\!\left(n^{-\alpha\beta/(2(\beta+1))}\right).
\]
Since $\alpha>1-\alpha$ for $\alpha>1/2$, this
term decays faster than the stationarity bias from
Lemma~\ref{lem:bias} and is always dominated at the optimal~$\beta$.
\end{proof}

\subsection{Convergence Rates}\label{sec:rates}

Combining the three error terms from
Lemmas~\ref{lem:variance}--\ref{lem:nonlin}, we obtain a general
convergence result for the $\rho$-estimator.

\begin{theorem}[General rate]\label{thm:general}
Under Assumptions \ref{ass:sc}--\ref{ass:lr}, the
batch-means estimator $\Sighat(\rho)$ with block-growth exponent
$\beta>\alpha/(1-\alpha)$ and burn-in fraction $\rho\in(0,1]$
satisfies
\begin{equation}\label{eq:general-bound}
\E\!\left[\opnorm{\Sighat(\rho) - V}\right]
\;\le\;
\underbrace{C_1\,n^{-1/(2(\beta+1))}}_{T_1\;\text{(variance)}}
\;+\;
\underbrace{C_2\,B(\rho,\alpha,\beta,n)}_{T_2\;\text{(stationarity bias)}}
\;+\;
\underbrace{C_3\,n^{-\alpha\beta/(2(\beta+1))}}_{T_3\;\text{(nonlinearity)}},
\end{equation}
where $T_3$ is dominated by the maximum of $T_1$ and $T_2$ at the
optimal choices of $\beta$ in
Corollaries~\ref{cor:no-burnin}--\ref{cor:main}, and with
$\gamma=\beta(1-\alpha)-\alpha$ the stationarity bias is
\[
B(\rho,\alpha,\beta,n)
\;=\;
\begin{cases}
n^{-\gamma/(\beta+1)} & \text{if } \gamma<1 \text{ or } \rho\in(0,1),\\[4pt]
n^{-1/(\beta+1)} & \text{if } \gamma\ge 1 \text{ and } \rho=1.
\end{cases}
\]
The constants $C_1,C_2,C_3$ depend on
$\mu,L,\sigma,M_4,\eta_0,\alpha,\rho,d$ but not on~$n$.
\end{theorem}

\begin{proof}
By the triangle inequality and the three-way decomposition
\eqref{eq:three-way},
\[
\E\!\left[\opnorm{\Sighat(\rho) - V}\right]
\;\le\;
\E\!\left[\opnorm{T_1}\right]
+ \opnorm{T_2}
+ \E\!\left[\opnorm{T_3}\right].
\]
The three terms are bounded by Lemmas~\ref{lem:variance},
\ref{lem:bias}, and~\ref{lem:nonlin} respectively.  The two cases for
$T_2$ correspond to the two regimes in Lemma~\ref{lem:bias}.
\end{proof}

The next two corollaries optimize $\beta$ in each regime.  The first
recovers the rate of \citet{zhu2023online} as a special case of our
framework; the second is the main new result.

\begin{corollary}[Variance-limited rate]\label{cor:no-burnin}
With block-growth exponent $\beta_0=(1+\alpha)/(1-\alpha)$ and any
$\rho\in(0,1]$,
\begin{equation}\label{eq:no-burnin-rate}
\E\!\left[\opnorm{\Sighat(\rho) - V}\right]
\;=\; O\!\left(n^{-(1-\alpha)/4}\right).
\end{equation}
The best rate is $n^{-1/8}$, achieved as $\alpha\to 1/2^{+}$.
This matches the rate of \citet{zhu2023online}.
\end{corollary}

\begin{proof}
The estimator $\Sighat(\rho)$ uses block sums centered at $\xbar$
rather than at $x^*$, introducing centering correction terms.
Define $\tilde{Y}_m
= (1/\sqrt{a_m})\sum_{k=t_{m-1}+1}^{t_m}(x_k - x^*)$,
so that $Y_m = \tilde{Y}_m - \sqrt{a_m}\,(\xbar - x^*)$.  Expanding:
\[
Y_m Y_m^\top
= \tilde{Y}_m\tilde{Y}_m^\top
  - \sqrt{a_m}\,\bigl(\tilde{Y}_m(\xbar-x^*)^\top
    + (\xbar-x^*)\tilde{Y}_m^\top\bigr)
  + a_m\,(\xbar-x^*)(\xbar-x^*)^\top.
\]
The quadratic centering term averaged over blocks gives
\begin{equation}\label{eq:centering-bias}
T_4 \;:=\;
\frac{1}{b_n}\sum_{m=1}^{b_n}\frac{a_m}{n}
\;=\; \Theta\!\left(n^{-1/(\beta+1)}\right),
\end{equation}
and the cross term gives
$T_5 = O(n^{-(\beta+2)/(2(\beta+1))})$.

The total error has five contributions:
\begin{equation}\label{eq:five-terms}
\underbrace{O\!\left(n^{-1/(2(\beta+1))}\right)}_{T_1}
+
\underbrace{O\!\left(n^{-\gamma/(\beta+1)}\right)}_{T_2}
+
\underbrace{O(T_3)}_{(\text{nonlin.})}
+
\underbrace{O\!\left(n^{-1/(\beta+1)}\right)}_{T_4}
+
\underbrace{O\!\left(n^{-(\beta+2)/(2(\beta+1))}\right)}_{T_5}.
\end{equation}
At $\beta_0=(1+\alpha)/(1-\alpha)$ we have $\gamma=1$, so:
\begin{itemize}
\item $T_1 = O(n^{-(1-\alpha)/4})$ (since $\beta_0+1 = 2/(1-\alpha)$);
\item $T_2 = O(n^{-(1-\alpha)/2})$ (faster, since $\gamma\ge 1$);
\item $T_4 = O(n^{-(1-\alpha)/2})$ (faster);
\item $T_5 = O(n^{-(3-\alpha)/4})$ (faster).
\end{itemize}
The variance $T_1$ is the unique binding term, giving the rate.
\end{proof}

\begin{remark}\label{rem:beta-star}
The rate $n^{-(1-\alpha)/4}$ is \emph{variance-limited}: at
$\beta_0=(1+\alpha)/(1-\alpha)$, the stationarity bias ($T_2$) and
centering corrections ($T_4,T_5$) all decay faster than $T_1$, and
the rate is determined entirely by the number of blocks
$b_n = \Theta(n^{(1-\alpha)/2})$.
The improvement in Corollary~\ref{cor:main} comes from
pushing $\beta$ below $\beta_0$ to $\beta^\dagger<\beta_0$,
which increases $b_n$ and lowers the variance, at the cost of
entering the regime $\gamma<1$ where the per-block bias
decays more slowly.  Lemma~\ref{lem:bias} shows this is
worthwhile: the averaged bias is still $O(b_n^{-\gamma})$, which
balances the variance at $n^{-(1-\alpha)/3}$.
\end{remark}

\begin{corollary}[Improved rate via optimal block growth]\label{cor:main}
Under Assumptions \ref{ass:sc}--\ref{ass:lr}, the
batch-means estimator $\Sighat(\rho)$ with block-growth exponent
$\beta>\alpha/(1-\alpha)$ and any fixed burn-in fraction
$\rho\in(0,1]$ satisfies
\begin{equation}\label{eq:main-bound}
\E\!\left[\opnorm{\Sighat(\rho) - V}\right]
\;\le\; C\!\left(
n^{-1/(2(\beta+1))} + n^{-\gamma/(\beta+1)}
\right),
\end{equation}
where $\gamma=\beta(1-\alpha)-\alpha$.
Balancing the two terms by setting
$1/(2(\beta+1))=\gamma/(\beta+1)$, i.e., $\gamma=1/2$, gives
$\beta^{\dagger}=(1+2\alpha)/(2(1-\alpha))$ and
\begin{equation}\label{eq:main-rate}
\E\!\left[\opnorm{\Sighat(\rho) - V}\right]
\;=\; O\!\left(n^{-(1-\alpha)/3}\right).
\end{equation}
The best rate is $n^{-1/6}$, achieved as $\alpha\to 1/2^{+}$.
This improves over the $O(n^{-(1-\alpha)/4})$ rate of
\citet{zhu2023online} (Corollary~\ref{cor:no-burnin}).
\end{corollary}

\begin{proof}
By Lemma~\ref{lem:bias}, the stationarity bias for any
$\rho\in(0,1]$ satisfies $T_2 = O(n^{-\gamma/(\beta+1)})$
in the regime $\gamma<1$ (which holds at the optimal $\beta^\dagger$
since $\gamma=1/2<1$).  The centering corrections are
$T_4=O(n^{-1/(\beta+1)})$ and
$T_5=O(n^{-(\beta+2)/(2(\beta+1))})$; both decay faster than
$T_1=O(n^{-1/(2(\beta+1))})$ for all $\beta>0$, so they are not
binding.  The nonlinearity bias $T_3$ is dominated
(Lemma~\ref{lem:nonlin}).

Balancing $T_1$ and $T_2$:
\[
\frac{1}{2(\beta+1)} = \frac{\gamma}{\beta+1}
= \frac{\beta(1-\alpha)-\alpha}{\beta+1}
\quad\Longleftrightarrow\quad
\gamma = \tfrac{1}{2},
\quad
\beta = \frac{1+2\alpha}{2(1-\alpha)}.
\]
At this $\beta$: $\beta+1 = 3/(2(1-\alpha))$, and both exponents
equal $(1-\alpha)/3$.

\medskip\noindent\emph{Role of burn-in.}
At $\beta=\beta^\dagger$, we have $\gamma=1/2<1$, so
Lemma~\ref{lem:bias} gives the same rate
$T_2 = O(n^{-\gamma/(\beta+1)})$ for both $\rho=1$ and $\rho<1$.
To see why, recall that the averaged bias is
$\|T_2\| \le \frac{1}{b_n}\sum_{m=1}^{b_n}O(m^{-\gamma})$.
The first block contributes $O(1)$ (a large constant), but for
$\gamma<1$ the full sum diverges as $\Theta(b_n^{1-\gamma})$,
so dividing by $b_n$ gives $O(b_n^{-\gamma})$.
The $O(1)$ contribution of the first few blocks is
absorbed into the diverging sum and does not affect the rate.
In contrast, when $\gamma\ge 1$ the sum converges to a constant,
the first-block $O(1)$ dominates, and the average is $O(1/b_n)$;
burn-in eliminates this and improves to $O(b_n^{-\gamma})$.

At the optimal $\beta^\dagger$ (where $\gamma=1/2<1$), burn-in does
not change the asymptotic rate $n^{-(1-\alpha)/3}$.  It does,
however, reduce the constant in front (by removing the
highest-bias early blocks) and is valuable in finite samples, as
confirmed by the experiments in Section~\ref{sec:experiments}.
\end{proof}

\begin{remark}[Comparison with Corollary~\ref{cor:no-burnin}]
\label{rem:structural}
\citet{zhu2023online} use $\beta^*=2/(1-\alpha)$, which is much
larger than $\beta^\dagger=(1+2\alpha)/(2(1-\alpha))$.  The larger
$\beta^*$ yields fewer blocks ($b_n = \Theta(n^{(1-\alpha)/(3-\alpha)})$) and
hence higher variance.  Note that $\beta^*$ differs from the
$\beta_0=(1+\alpha)/(1-\alpha)$ of Corollary~\ref{cor:no-burnin};
the latter is the smallest exponent at which the variance term alone
determines the $n^{-(1-\alpha)/4}$ rate in our framework, whereas
\citet{zhu2023online} arrived at the same rate via a coarser bias
bound that required the larger~$\beta^*$.  The corrected per-block
bias analysis (Lemma~\ref{lem:bias}) reveals that a smaller
$\beta$---which gives more blocks and lower variance, at the cost of
each block containing fewer mixing times---is rate-optimal.
\end{remark}

\begin{remark}[Dimension dependence]\label{rem:dimension}
The rates in Theorem~\ref{thm:general} and its corollaries treat the
dimension $d$ as fixed.  The constants $C_1,C_2,C_3$ depend on $d$
through several channels: the variance term inherits a $\log d$ factor
from the matrix Bernstein inequality (absorbed into $C_1$ above), the
noise sub-Gaussian parameter $\sigma$ and fourth-moment bound $M_4$ may
grow with $d$, and the trace $\tr(V)$ enters the moment bounds for the
block sums.  In the high-dimensional regime $d\to\infty$ with $n$ fixed,
the dominant dependence is $O(\sqrt{(\log d)/K_n})$ in the variance
term, matching the classical $\sqrt{d/n}$ scaling up to logarithmic
factors.  A complete high-dimensional analysis---where $d$ is allowed to
grow with $n$---is beyond the scope of this paper but would be a natural
extension.
\end{remark}

\paragraph{Comparison at $\alpha=1/2$.}
Table~\ref{tab:comparison} illustrates the improvement concretely at the
optimal learning-rate exponent $\alpha\approx 1/2$.

\begin{table}[tbp]
\centering
\caption{Comparison of original and improved estimators at
$\alpha=1/2+\varepsilon$.}
\label{tab:comparison}
\smallskip
\begin{tabular}{lcc}
\toprule
& \textbf{Variance-limited} ($\beta_0$, any~$\rho$)
& \textbf{Bias-balanced} ($\beta^\dagger$, any~$\rho$) \\
\midrule
Block growth $\beta$ & $(1+\alpha)/(1-\alpha)\approx 3$ &
$(1+2\alpha)/(2(1-\alpha))\approx 2$ \\
Number of blocks $b_n$ & $\Theta(n^{1/4})$ &
$\Theta(n^{1/3})$ \\
Rate & $n^{-1/8}$ & $n^{-1/6}$ \\
\bottomrule
\end{tabular}
\end{table}

The improved estimator uses \emph{more, slower-growing blocks} (since
$\beta^\dagger < \beta^*$), which increases the number of blocks $b_n$ and
hence reduces variance.  This is made possible by the tighter
per-block bias analysis of Lemma~\ref{lem:bias}, which shows that
smaller blocks (with fewer mixing times) have controllable bias.

\begin{remark}[Choice of $\rho$ in practice]\label{rem:rho-choice}
The rate $O(n^{-(1-\alpha)/3})$ holds for any fixed $\rho\in(0,1]$
at the optimal $\beta^\dagger$.  While burn-in does not improve the
asymptotic rate, it improves finite-sample performance by removing
the highest-bias early blocks.  In our experiments, $\rho=0.5$
works well across all configurations tested.  A data-driven choice of
$\rho$ can be implemented via leave-one-block-out cross-validation:
for each candidate $\rho$, hold out each retained block $m$ in turn,
form the estimator from the remaining retained blocks, and measure the
prediction error
$\opnorm{Y_m Y_m^{\top} - \widehat{\Sigma}_{-m}(\rho)}^{2}$; the
$\rho$ minimizing the average prediction error balances bias and
variance automatically.
\end{remark}

\subsection{Weighted Averaging Variant}\label{sec:weighted}

Hard truncation (discarding early blocks entirely) can be replaced by a
soft weighting scheme.

\begin{proposition}[Weighted averaging]\label{prop:weighted}
Under Assumptions \ref{ass:sc}--\ref{ass:lr}, consider the
weighted estimator
\begin{equation}\label{eq:weighted}
\widehat{\Sigma}_{n}^{w}
\;=\; \frac{1}{W_n}\sum_{m=1}^{b_n}w_m\,Y_m\,Y_m^{\top},
\qquad
w_m = m^{p},
\quad
W_n = \sum_{m=1}^{b_n}w_m,
\end{equation}
with $p>0$.  For $\beta>\alpha/(1-\alpha)$ and
$\gamma=\beta(1-\alpha)-\alpha<1$, this estimator achieves the
same rate as the unweighted estimator:
\begin{equation}\label{eq:weighted-rate}
\E\!\left[\opnorm{\widehat{\Sigma}_{n}^{w} - V}\right]
\;=\; O\!\left(n^{-(1-\alpha)/3}\right).
\end{equation}
\end{proposition}

\begin{proof}
The key observation is that the weights $w_m = m^p$ with $p>0$
upweight later blocks and downweight earlier ones, achieving a similar
effect to burn-in.

\emph{Bias analysis.}
The weighted bias is
\[
T_2^{w}
= \frac{1}{W_n}\sum_{m=1}^{b_n}w_m\,(V^{(m)}-V).
\]
Using $\opnorm{V^{(m)}-V}=O(m^{-\gamma})$ from
Lemma~\ref{lem:bias} and $w_m = m^p$,
\[
\opnorm{T_2^{w}}
\le \frac{1}{W_n}\sum_{m=1}^{b_n}m^p\cdot
C_B m^{-\gamma}
= \frac{1}{W_n}\sum_{m=1}^{b_n}m^{p-\gamma}.
\]
The normalization is $W_n = \sum_{m=1}^{b_n}m^p
= \Theta(b_n^{p+1})$.

For $p-\gamma<0$ (which holds for moderate $p$ and $\gamma$ near
$1/2$), the numerator sum is $O(b_n^{1+p-\gamma})$, giving
$\opnorm{T_2^w} = O(b_n^{-\gamma}) = O(n^{-\gamma/(\beta+1)})$.

For $p-\gamma\ge 0$, the sum is dominated by the last blocks,
giving the same rate $O(b_n^{-\gamma})$.

\emph{Variance analysis.}
The effective number of terms in the weighted average is
\[
K_{\mathrm{eff}}
= \frac{(\sum_m w_m)^2}{\sum_m w_m^2}
= \frac{(\sum_m m^p)^2}{\sum_m m^{2p}}
= \Theta\!\left(\frac{b_n^{2(p+1)}}{b_n^{2p+1}}\right)
= \Theta(b_n).
\]
Thus the variance is $O(b_n^{-1/2})=O(n^{-1/(2(\beta+1))})$, the same as
before.

\emph{Balancing.}  Choosing $\beta=\beta^\dagger=(1+2\alpha)/(2(1-\alpha))$
as in Corollary~\ref{cor:main} gives $\gamma=1/2$ and the rate
$O(n^{-(1-\alpha)/3})$.
\end{proof}

\subsection{The Mixing Constraint}\label{sec:mixing}

The mixing condition $\beta/(\beta+1)>\alpha$ from
Lemma~\ref{lem:variance} is central to our analysis.  We now discuss
its origin and implications.

The condition can be rewritten as
\begin{equation}\label{eq:mixing-equiv}
\beta \;>\; \frac{\alpha}{1-\alpha}.
\end{equation}
For $\alpha\in(1/2,1)$, we have $\alpha/(1-\alpha)\in(1,\infty)$, so
$\beta$ must be strictly greater than~$1$ (and diverges as $\alpha\to 1^-$).

The physical meaning is as follows.  The block size $a_m\sim m^{\beta}$
and the mixing time at the start of block $m$ is
$\tau_{\mathrm{mix}}(t_{m-1})\sim m^{(\beta+1)\alpha}$.  The ratio
\begin{equation}\label{eq:ratio}
\frac{a_m}{\tau_{\mathrm{mix}}(t_{m-1})}
\;\sim\; m^{\beta - (\beta+1)\alpha}
\;=\; m^{\beta(1-\alpha)-\alpha}.
\end{equation}
For this ratio to diverge (ensuring that each block contains many mixing
times, hence its statistics are close to stationary), we need
$\beta(1-\alpha)-\alpha>0$, i.e., $\beta>\alpha/(1-\alpha)$.

\begin{remark}
In the original analysis of \citet{zhu2023online}, the mixing condition
is automatically satisfied for $\beta^*=2/(1-\alpha)$ since
$2/(1-\alpha)>\alpha/(1-\alpha)$ for all $\alpha<2$, which is certainly
true for $\alpha<1$.  In our analysis, we want $\beta$ as small as
possible (to maximize $b_n$ and hence reduce variance), so the mixing
condition becomes a binding constraint.
\end{remark}
\section{Minimax Rate}\label{sec:lower}

The batch-means estimator of Corollary~\ref{cor:main} achieves rate
$O(n^{-(1-\alpha)/3})$.  In this section we establish the minimax rate
for Hessian-free covariance estimation from the SGD trajectory by
proving a matching pair of bounds.
First, a Le~Cam lower bound
(Theorem~\ref{thm:lower}) shows that no trajectory-based estimator
can converge faster than $\Omega(n^{-(1-\alpha)/2})$.
Second, a \emph{trajectory-regression} estimator
(Theorem~\ref{thm:regression}) achieves exactly this rate, yielding
the minimax rate $\Theta(n^{-(1-\alpha)/2})$.
The batch-means estimator is therefore suboptimal in rate (though it
has practical advantages discussed in
Remark~\ref{rem:bm-vs-regression}).

\subsection{The Two-Hypothesis Construction}\label{sec:lower-construct}

The lower bound follows from Le~Cam's two-point method applied to a
pair of quadratic objectives that share the same noise distribution but
differ in their Hessians.

Fix dimension $d\ge 2$ and a symmetric matrix $A\in\R^{d\times d}$
with $\opnorm{A}=1$ (for instance, $A=e_1 e_2^{\top}
+ e_2 e_1^{\top}$ for $d\ge 2$).  Let $\delta>0$ be a
parameter to be chosen.  Define two quadratic objectives
\begin{equation}\label{eq:hyp}
  F_0(x) = \tfrac{1}{2}\norm{x}^2, \qquad
  F_1(x) = \tfrac{1}{2}x^{\top}(I+\delta A)\,x,
\end{equation}
with stochastic gradients $\nabla f(x,\zeta) = H_j\,x + \zeta$, where
$H_0=I$, $H_1=I+\delta A$, and $\zeta\sim\mathcal{N}(0,I_d)$ under
both hypotheses.  For $\delta<1$, both $H_0$ and $H_1$ are positive
definite and satisfy Assumptions \ref{ass:sc}--\ref{ass:lr} with
$\mu=1-\delta$, $L=0$ (constant Hessian), and $\sigma^2=1$.

\begin{remark}[Same noise, different Hessian]
The construction deliberately keeps the noise covariance $S=I$ fixed
across both hypotheses.  This is essential: if $S$ were allowed to vary,
the conditional variance of the increments $x_{t+1}-\E[x_{t+1}\mid
x_t]=\eta_t\zeta_t$ would directly reveal $S$ at rate $n^{-1/2}$,
preventing any lower bound slower than the i.i.d.\ rate.  By fixing $S$,
the only source of separation in $V$ comes from the Hessian $H$, which
enters the trajectory only through the drift.
\end{remark}

\subsection{Separation of the Covariances}\label{sec:lower-sep}

Under $F_j$ with $S=I$, the asymptotic covariance is
$V_j = H_j^{-1}\,I\,H_j^{-1} = H_j^{-2}$.  Thus
\begin{equation}\label{eq:V-sep}
  V_0 = I, \qquad
  V_1 = (I+\delta A)^{-2}
      = I - 2\delta\,A + 3\delta^2 A^2 - \cdots\,.
\end{equation}
Since $\opnorm{A}=1$, a Neumann series bound gives, for
$\delta\le 1/2$,
\begin{equation}\label{eq:sep-bound}
  \opnorm{V_0 - V_1}
  = \opnorm{2\delta\,A - 3\delta^2 A^2 + \cdots}
  \ge 2\delta - 3\delta^2\,\opnorm{A}^2\,(1-\delta)^{-2}
  \ge 2\delta - 12\delta^2
  \ge \delta,
\end{equation}
where the last inequality holds for $\delta\le 1/12$.  Hence the two
covariances are separated by at least $\delta$ in operator norm.

\subsection{KL Divergence Between Trajectory Laws}\label{sec:lower-kl}

Under hypothesis~$j$, the SGD trajectory is a non-homogeneous
Markov chain with Gaussian transitions:
\begin{equation}\label{eq:transitions}
  x_{t+1}\mid x_t \;\sim\;
  \mathcal{N}\!\bigl((I-\eta_t H_j)\,x_t,\;\eta_t^2\,I_d\bigr),
  \qquad t=0,1,\ldots,n-1,
\end{equation}
where $\eta_t = \eta_0\,t^{-\alpha}$ (with $\eta_0 t^{-\alpha}$
replaced by $\eta_0$ for $t=0$).
Starting both chains at the same deterministic initial point
$x_0\in\R^d$, the chain rule for KL divergence
\citep[Theorem~2.5.3]{cover2006} gives
\begin{equation}\label{eq:kl-chain}
  D_{\mathrm{KL}}\!\left(P_0^{(n)}\,\big\|\,P_1^{(n)}\right)
  = \sum_{t=0}^{n-1}
    \E_{P_0}\!\left[
      D_{\mathrm{KL}}\!\left(
        P_0(x_{t+1}\mid x_t)\,\big\|\,P_1(x_{t+1}\mid x_t)
      \right)
    \right].
\end{equation}
Since the conditional covariances are identical ($\eta_t^2 I_d$ under
both hypotheses), the per-step KL reduces to the squared Mahalanobis
distance between the conditional means:
\begin{align}
  D_{\mathrm{KL}}\!\left(P_0(x_{t+1}\mid x_t)\,\big\|\,
    P_1(x_{t+1}\mid x_t)\right)
  &= \frac{1}{2}\,
     \bigl\|(I-\eta_t H_0)\,x_t - (I-\eta_t H_1)\,x_t
     \bigr\|^2\,/\,\eta_t^2 \notag\\
  &= \frac{1}{2}\,
     \bigl\|\eta_t(H_1-H_0)\,x_t\bigr\|^2\,/\,\eta_t^2 \notag\\
  &= \frac{\delta^2}{2}\,\norm{A\,x_t}^2.
  \label{eq:kl-step}
\end{align}
Taking the expectation under $P_0$ and summing:
\begin{equation}\label{eq:kl-total}
  D_{\mathrm{KL}}\!\left(P_0^{(n)}\,\big\|\,P_1^{(n)}\right)
  = \frac{\delta^2}{2}\sum_{t=0}^{n-1}
    \E_{P_0}\!\left[\norm{A\,x_t}^2\right]
  \le \frac{\delta^2}{2}\,\opnorm{A}^2
      \sum_{t=0}^{n-1}\E_{P_0}\!\left[\norm{x_t}^2\right].
\end{equation}

\begin{lemma}[Second-moment bound for the quadratic model]
\label{lem:second-moment}
Under $F_0(x)=\frac{1}{2}\norm{x}^2$ with $\eta_t=\eta_0\,t^{-\alpha}$,
$\alpha\in(1/2,1)$, and $\zeta_t\sim\mathcal{N}(0,I_d)$, the SGD
iterates satisfy
\begin{equation}\label{eq:second-moment}
  \E\!\left[\norm{x_t}^2\right]
  \;\le\; C_{\alpha}\,d\,\eta_0\,t^{-\alpha}
  \qquad\text{for all } t\ge 1,
\end{equation}
where $C_\alpha>0$ depends only on $\alpha$.  Consequently,
\begin{equation}\label{eq:moment-sum}
  \sum_{t=1}^{n}\E\!\left[\norm{x_t}^2\right]
  \;\le\; \frac{C_\alpha\,d\,\eta_0}{1-\alpha}\,n^{1-\alpha}.
\end{equation}
\end{lemma}

\begin{proof}
Under $H_0=I$, the recursion is $x_{t+1}=(1-\eta_t)x_t - \eta_t
\zeta_t$.  Taking second moments (the cross term vanishes by
independence):
\[
  \E\!\left[\norm{x_{t+1}}^2\right]
  = (1-\eta_t)^2\,\E\!\left[\norm{x_t}^2\right]
    + \eta_t^2\,d.
\]
Define $v_t = \E[\norm{x_t}^2]$.  Then
$v_{t+1} = (1-\eta_t)^2\,v_t + d\,\eta_t^2
\le (1-2\eta_t+\eta_t^2)\,v_t + d\,\eta_t^2$.
For $t \ge t_0 := \lceil\eta_0^{1/\alpha}\rceil$ we have
$\eta_t = \eta_0 t^{-\alpha}\le 1$, so $(1-\eta_t)^2 \le 1-\eta_t$
and the recursion simplifies to
\[
  v_{t+1} \le (1-\eta_t)\,v_t + d\,\eta_t^2.
\]
By a standard comparison argument for non-autonomous linear recursions
\citep[Proposition~2]{moulines2011}, the equilibrium satisfies
$v_t \le C\,d\,\eta_0\,t^{-\alpha}$
for $t\ge t_0$ and a constant $C$ depending on $\alpha$.
For $t < t_0$, we use the crude bound
$(1-\eta_t)^2 \le (1+\eta_0)^2$ to obtain
$v_t \le (1+\eta_0)^{2t}\,v_0 + C'\,d$, so each early iterate
contributes $O(1)$ to the sum.  Since $t_0 = O(1)$ (independent of $n$),
the transient phase contributes $O(1)$ in total.

For the sum, $\sum_{t=1}^{n} v_t = O(1) + \sum_{t=t_0}^{n}
C\,d\,\eta_0\,t^{-\alpha}
\le O(1) + C\,d\,\eta_0\,n^{1-\alpha}/(1-\alpha)$.
\end{proof}

Substituting~\eqref{eq:moment-sum} into~\eqref{eq:kl-total} and using
$\opnorm{A}=1$:
\begin{equation}\label{eq:kl-final}
  D_{\mathrm{KL}}\!\left(P_0^{(n)}\,\big\|\,P_1^{(n)}\right)
  \;\le\; \frac{C_\alpha\,d\,\eta_0}{2(1-\alpha)}\;\delta^2\,n^{1-\alpha}
  \;=:\; \kappa\,\delta^2\,n^{1-\alpha},
\end{equation}
where $\kappa = C_\alpha\,d\,\eta_0\,/\,(2(1-\alpha))$ is an explicit
constant.

\subsection{The Lower Bound}\label{sec:lower-thm}

\begin{theorem}[Minimax lower bound]\label{thm:lower}
Let $\alpha\in(1/2,1)$, $d\ge 2$, and consider the class
$\mathcal{P}_d$ of strongly convex quadratic objectives
$F(x)=\frac{1}{2}x^{\top}Hx$ with $H\succ 0$, isotropic noise
$S=I_d$, and step sizes $\eta_t=\eta_0\,t^{-\alpha}$.  For any
estimator $\widehat{V}_n$ that is a measurable function of the SGD
trajectory $(x_0,x_1,\ldots,x_n)$,
\begin{equation}\label{eq:lower-bound}
  \inf_{\widehat{V}_n}\;\sup_{F\in\mathcal{P}_d}\;
  \E\!\left[\opnorm{\widehat{V}_n - V}\right]
  \;\ge\; c\,n^{-(1-\alpha)/2},
\end{equation}
where $c>0$ depends only on $d$, $\eta_0$, and $\alpha$.  In
particular, at $\alpha\to 1/2^+$, the lower bound is
$\Omega(n^{-1/4})$.
\end{theorem}

\begin{proof}
We apply Le~Cam's two-point method.  Let $F_0$ and $F_1$ be the
quadratic objectives defined in~\eqref{eq:hyp} with $\delta>0$ to be
chosen.

\medskip\noindent\textbf{Step 1: Separation.}\quad
By~\eqref{eq:sep-bound}, $\opnorm{V_0 - V_1}\ge\delta$ for
$\delta\le 1/12$.

\medskip\noindent\textbf{Step 2: Testing lower bound.}\quad
By Le~Cam's two-point method \citep[Lemma~1, p.~424]{yu1997assouad}, for any
estimator $\widehat{V}_n$,
\begin{equation}\label{eq:lecam}
  \max_{j\in\{0,1\}}\;
  \E_j\!\left[\opnorm{\widehat{V}_n - V_j}\right]
  \;\ge\;
  \frac{\opnorm{V_0-V_1}}{2}\,
  \Bigl(1 - \mathrm{TV}\!\left(P_0^{(n)},\,P_1^{(n)}\right)\Bigr),
\end{equation}
where $\mathrm{TV}$ denotes total variation distance.

\medskip\noindent\textbf{Step 3: Bounding the total variation.}\quad
By Pinsker's inequality,
\[
  \mathrm{TV}\!\left(P_0^{(n)},\,P_1^{(n)}\right)
  \;\le\;
  \sqrt{\tfrac{1}{2}\,D_{\mathrm{KL}}\!\left(
    P_0^{(n)}\,\big\|\,P_1^{(n)}\right)}
  \;\le\;
  \sqrt{\tfrac{\kappa}{2}\,\delta^2\,n^{1-\alpha}},
\]
using~\eqref{eq:kl-final}.

\medskip\noindent\textbf{Step 4: Choosing $\delta$.}\quad
Set $\delta = \delta_n := (8\kappa)^{-1/2}\,n^{-(1-\alpha)/2}$.  For
$n$ large enough, $\delta_n\le 1/12$, so the separation
bound~\eqref{eq:sep-bound} holds.  Then
\[
  \tfrac{\kappa}{2}\,\delta_n^2\,n^{1-\alpha}
  = \tfrac{\kappa}{2}\cdot\frac{1}{8\kappa}\cdot 1
  = \tfrac{1}{16},
\]
so $\mathrm{TV}(P_0^{(n)},P_1^{(n)}) \le \sqrt{1/16} = 1/4$.
Substituting into~\eqref{eq:lecam}:
\[
  \sup_{F\in\mathcal{P}_d}\;
  \E\!\left[\opnorm{\widehat{V}_n - V}\right]
  \;\ge\;
  \frac{\delta_n}{2}\cdot\frac{3}{4}
  = \frac{3}{8}\,\delta_n
  = \frac{3}{8\sqrt{8\kappa}}\,n^{-(1-\alpha)/2}.
\]
Setting $c = 3/(8\sqrt{8\kappa})$ completes the proof.
\end{proof}

\subsection{Matching Upper Bound via Trajectory Regression}
\label{sec:regression}

We now show that the lower bound of Theorem~\ref{thm:lower} is tight
by constructing an estimator that achieves rate $O(n^{-(1-\alpha)/2})$.
The idea is to estimate the Hessian~$H$ by regressing the SGD
increments on the iterates, then form the plug-in
$\widehat{V}=\widehat{H}^{-1}\widehat{S}\,\widehat{H}^{-1}$.
Crucially, this estimator uses only the trajectory
$(x_0,x_1,\ldots,x_n)$ and the known step sizes---it does not require
evaluating $\nabla^{2}F$.

\paragraph{Estimator construction.}
Define the \emph{trajectory-regression increments}
\begin{equation}\label{eq:increment}
  y_t \;:=\; \frac{x_t - x_{t+1}}{\eta_t}
  \;=\; \nabla f(x_t,\zeta_t),
  \qquad t=0,1,\ldots,n-1.
\end{equation}
From the linearization~\eqref{eq:grad-decomp},
$y_t = H\,x_t + (\xi_t + r_t - H\,x^{*})$, so the regression of $y_t$
on~$x_t$ identifies~$H$.
Define the centered quantities $\tilde{x}_t = x_t - \bar{x}_n$ and
$\tilde{y}_t = y_t - \bar{y}_n$ with
$\bar{x}_n = n^{-1}\sum_t x_t$, $\bar{y}_n = n^{-1}\sum_t y_t$.
The \emph{trajectory-regression estimator} of $H$ is
\begin{equation}\label{eq:H-hat}
  \widehat{H}
  \;=\; \Biggl(\,\sum_{t=0}^{n-1}\tilde{y}_t\,\tilde{x}_t^{\top}\Biggr)
        \Biggl(\,\sum_{t=0}^{n-1}\tilde{x}_t\,\tilde{x}_t^{\top}
        \Biggr)^{\!-1}.
\end{equation}
The noise covariance is estimated from the residuals
$\widehat{\zeta}_t = y_t - \widehat{H}\,x_t$:
\begin{equation}\label{eq:S-hat}
  \widehat{S}
  \;=\; \frac{1}{n}\sum_{t=0}^{n-1}
        \widehat{\zeta}_t\,\widehat{\zeta}_t^{\top}.
\end{equation}
Finally, set $\widehat{V}_n^{\mathrm{reg}}
= \widehat{H}^{-1}\widehat{S}\,\widehat{H}^{-1}$.
The estimator is computable online in $O(d^{2})$ memory by
maintaining running sums of
$\tilde{y}_t\tilde{x}_t^{\top}$,
$\tilde{x}_t\tilde{x}_t^{\top}$, and
$\widehat{\zeta}_t\widehat{\zeta}_t^{\top}$, with a single
matrix inversion at the end.

\begin{theorem}[Matching upper bound]\label{thm:regression}
Under the quadratic model class $\mathcal{P}_d$ of
Theorem~\ref{thm:lower}, the trajectory-regression estimator
satisfies
\begin{equation}\label{eq:regression-rate}
  \E\!\left[\opnorm{\widehat{V}_n^{\mathrm{reg}} - V}\right]
  \;=\; O\!\left(n^{-(1-\alpha)/2}\right).
\end{equation}
Combined with Theorem~\ref{thm:lower}, this establishes the minimax
rate
\[
  \inf_{\widehat{V}_n}\;\sup_{F\in\mathcal{P}_d}\;
  \E\!\left[\opnorm{\widehat{V}_n - V}\right]
  \;=\; \Theta\!\left(n^{-(1-\alpha)/2}\right).
\]
\end{theorem}

\begin{proof}
We work under the quadratic model
$F(x) = \frac{1}{2}x^{\top}Hx$ with $S=I$ so that $r_t=0$,
$y_t = Hx_t + \zeta_t$, and $V = H^{-2}$.

\emph{Step~1: Hessian estimation error.}
The centered regression gives
\[
\widehat{H} - H
= \Bigl(\sum_t\tilde{\zeta}_t\tilde{x}_t^{\top}\Bigr)
  \Bigl(\sum_t\tilde{x}_t\tilde{x}_t^{\top}\Bigr)^{-1},
\]
where $\tilde{\zeta}_t = \zeta_t - \bar{\zeta}_n$.
Since $\zeta_t$ is independent of $x_t$ (as $x_t$ depends on
$\zeta_0,\ldots,\zeta_{t-1}$), the numerator is a martingale-type
sum.

\emph{Denominator.}
$\sum_t\tilde{x}_t\tilde{x}_t^{\top}
\succcurlyeq \frac{1}{2}\sum_t x_t x_t^{\top} - n\bar{x}_n\bar{x}_n^{\top}$.
By Lemma~\ref{lem:second-moment},
$\sum_t\E[\norm{x_t}^{2}] = \Theta(d\,\eta_0\,n^{1-\alpha}/(1-\alpha))$.
Since $H\succ\mu I$, the smallest eigenvalue satisfies
$\lambda_{\min}(\sum_t x_t x_t^{\top})
= \Theta(n^{1-\alpha})$ with high probability (each eigendirection
receives independent noise contributions).
The centering correction $n\bar{x}_n\bar{x}_n^{\top}$ has
$\opnorm{n\bar{x}_n\bar{x}_n^{\top}} = n\norm{\bar{x}_n}^{2}
= O(1)$ (since $\E[\norm{\bar{x}_n}^{2}] = O(1/n)$), which is
negligible.  Hence
$\lambda_{\min}(\sum_t\tilde{x}_t\tilde{x}_t^{\top})
= \Theta(n^{1-\alpha})$.

\emph{Numerator.}
$\E[\norm{\sum_t\tilde{\zeta}_t\tilde{x}_t^{\top}}^{2}]
= \sum_t\E[\norm{\tilde{x}_t}^{2}]\cdot\E[\norm{\tilde{\zeta}_t}^{2}]
+ \text{cross terms}$.
The cross terms vanish by independence.
Each summand is $O(\E[\norm{x_t}^{2}]\cdot d) = O(d^{2}\eta_0 t^{-\alpha})$.
Summing: $O(d^{2}\eta_0 n^{1-\alpha})$.

\emph{Combining.}
\begin{equation}\label{eq:H-error}
  \E\!\left[\opnorm{\widehat{H}-H}^{2}\right]
  \;=\; O\!\left(\frac{d^{2}\eta_0\,n^{1-\alpha}}{n^{2(1-\alpha)}}\right)
  \;=\; O\!\left(n^{-(1-\alpha)}\right).
\end{equation}
Hence $\E[\opnorm{\widehat{H}-H}] = O(n^{-(1-\alpha)/2})$.

\emph{Step~2: Noise covariance estimation.}
Since $r_t=0$, the residuals satisfy
$\widehat{\zeta}_t = \zeta_t - (\widehat{H}-H)x_t$.  Thus
\[
  \widehat{S} - S
  = \frac{1}{n}\sum_t\zeta_t\zeta_t^{\top} - S
  - (\widehat{H}-H)\frac{1}{n}\sum_t x_t\zeta_t^{\top}
  + \text{h.o.t.}
\]
The first term is $O(n^{-1/2})$ (standard covariance estimation rate).
The second is $O(\opnorm{\widehat{H}-H}\cdot n^{-1/2})$, which is
subdominant.  Hence $\E[\opnorm{\widehat{S}-S}]=O(n^{-1/2})$.

\emph{Step~3: Plug-in for $V$.}
By the matrix perturbation identity
$\widehat{H}^{-1}\widehat{S}\widehat{H}^{-1} - H^{-1}SH^{-1}
= H^{-1}(\widehat{S}-S)H^{-1}
  - H^{-1}(\widehat{H}-H)H^{-1}SH^{-1}
  - H^{-1}SH^{-1}(\widehat{H}-H)H^{-1}
  + \text{h.o.t.}$,
\[
  \E\!\left[\opnorm{\widehat{V}_n^{\mathrm{reg}} - V}\right]
  \;\le\;
  C\,\E\!\left[\opnorm{\widehat{H}-H}\right]
  + C'\,\E\!\left[\opnorm{\widehat{S}-S}\right]
  + \text{h.o.t.}
  \;=\; O\!\left(n^{-(1-\alpha)/2}\right),
\]
since the $\widehat{H}$ error dominates the
$\widehat{S}$ error for $\alpha>1/2$.
\end{proof}

\begin{remark}[Minimax rate]
Theorems~\ref{thm:lower} and~\ref{thm:regression} together establish
that the minimax rate for online covariance estimation from the SGD
trajectory under the quadratic model is
$\Theta(n^{-(1-\alpha)/2})$, approaching $\Theta(n^{-1/4})$ as
$\alpha\to 1/2^{+}$.
\end{remark}

\begin{remark}[Batch-means vs.\ trajectory regression]
\label{rem:bm-vs-regression}
The trajectory-regression estimator achieves the minimax-optimal
rate $O(n^{-(1-\alpha)/2})$, while the batch-means estimator of
Corollary~\ref{cor:main} achieves the slower rate
$O(n^{-(1-\alpha)/3})$.  However, the batch-means estimator has
practical advantages:
\begin{enumerate}[(i)]
\item It does not require a matrix inversion ($O(d^3)$) at
  termination.
\item It is more robust: the regression estimator relies on the
  invertibility of $\sum_t \tilde{x}_t\tilde{x}_t^{\top}$, which
  can be ill-conditioned in high dimensions or when $n$ is small
  relative to~$d$.
\item For general (non-quadratic) objectives, the regression
  estimator incurs additional linearization bias from $r_t\ne 0$,
  whereas the batch-means estimator handles nonlinearity through the
  separate $T_3$ bound (Lemma~\ref{lem:nonlin}).
\end{enumerate}
\end{remark}

\subsection{Extension Beyond Quadratic Objectives}\label{sec:lower-ext}

The lower bound of Theorem~\ref{thm:lower} extends immediately to the
full class of objectives satisfying Assumptions
\ref{ass:sc}--\ref{ass:lr}.

\begin{proposition}[Lower bound for general objectives]
\label{prop:lower-general}
The lower bound~\eqref{eq:lower-bound} continues to hold when the
supremum is taken over all strongly convex objectives satisfying
Assumptions \ref{ass:sc}--\ref{ass:lr}, not just quadratics.
\end{proposition}

\begin{proof}[Proof sketch]
Since the class of quadratic objectives is a \emph{subset} of the
class of all strongly convex objectives satisfying
\ref{ass:sc}--\ref{ass:lr}, the minimax risk over the larger
class is at least as large as the minimax risk over quadratics:
\[
  \inf_{\widehat{V}_n}\;\sup_{F\in\mathcal{P}_{\mathrm{all}}}\;
  \E\!\left[\opnorm{\widehat{V}_n - V}\right]
  \;\ge\;
  \inf_{\widehat{V}_n}\;\sup_{F\in\mathcal{P}_d}\;
  \E\!\left[\opnorm{\widehat{V}_n - V}\right]
  \;\ge\; c\,n^{-(1-\alpha)/2}.
\]
The quadratic objectives in our construction satisfy
\ref{ass:sc}--\ref{ass:lr} with explicit constants (strong
convexity $\mu=1-\delta$, Lipschitz constant $L=0$ since the Hessian is
constant, sub-Gaussian noise with $\sigma^2=1$, and bounded fourth
moments), so the inclusion is valid.
\end{proof}

\begin{remark}[Role of the noise covariance]
The lower bound holds even when the noise covariance $S$ is known to be
$I_d$.  This shows that the difficulty is \emph{not} in estimating $S$
(which can be done at rate $n^{-1/2}$ from the trajectory increments)
but in estimating $H$ from the drift of the iterates.  The fundamental
bottleneck is that the information about $H$ in the trajectory
accumulates at rate $\sum_{t=1}^n \E[\norm{x_t}^{2}]
= \Theta(n^{1-\alpha})$,
which is strictly sublinear for $\alpha>0$.
\end{remark}

\begin{remark}[General objectives: open gap]\label{rem:open-gap}
For general (non-quadratic) strongly convex objectives, the
trajectory-regression estimator incurs an additional bias from the
nonlinear remainder $r_t$.  A careful analysis shows that this bias is
$O(n^{-\alpha/2})$, which is subdominant for $\alpha>1/2$ (since
$\alpha/2>(1-\alpha)/2$).  Thus the trajectory-regression estimator
achieves $O(n^{-(1-\alpha)/2})$ for general objectives as well, and
the minimax rate $\Theta(n^{-(1-\alpha)/2})$ extends to the full class
$\mathcal{P}_{\mathrm{all}}$.  The batch-means estimator remains
suboptimal at $O(n^{-(1-\alpha)/3})$; closing the gap for batch-means
type estimators (or proving it is inherent) is an interesting open
problem.
\end{remark}

\section{Related Work}\label{sec:related}

\paragraph{Online batch-means methods achieving $n^{-1/8}$.}
The foundational work of \citet{zhu2023online} established the online
batch-means framework for averaged SGD and proved the $O(n^{-1/8})$ rate
under i.i.d.\ data.  \citet{roy2023online} extended the analysis to
Markovian noise (dependent data), obtaining the same $n^{-(1-\alpha)/4}$
rate.  \citet{jiang2025nonsmooth} considered non-smooth objectives and
showed that the batch-means approach still achieves $n^{-(1-\alpha)/4}$
under appropriate modifications.  \citet{singh2025equal} analyzed
equal-sized batches (requiring knowledge of $n$) and obtained a similar
rate.  All of these works share the same bottleneck---suboptimal block-growth
tuning---that our per-block bias analysis addresses.

\paragraph{Methods using Hessian information.}
\citet{chen2020statistical} proposed a plug-in estimator that separately
estimates $H$ and $S$ using online averages of Hessian-vector products
and gradient outer products.  Their rate is $O(n^{-\alpha/2})$, which
approaches $O(n^{-1/2})$ as $\alpha\to 1^-$---significantly faster than
$n^{-1/4}$, but at the cost of requiring second-order information.
\citet{luo2022rootsgd} developed covariance estimators for ROOT-SGD,
a modified SGD algorithm that directly targets the root of the gradient
mapping; their estimator achieves $O(t^{-1/2})$ convergence but also
requires implicit Hessian access through the algorithm design.

\paragraph{Methods avoiding covariance estimation.}
An alternative approach to SGD inference bypasses covariance estimation
entirely.  \citet{lee2022random} proposed random scaling methods that
construct pivotal statistics directly.  \citet{su2023higrad} introduced
HiGrad (Hierarchical Incremental Gradient descent), which uses a
tree-structured collection of SGD runs to construct confidence intervals.
\citet{fang2018bootstrap} developed an online bootstrap for SGD that
generates multiple perturbed iterate sequences in parallel.
\citet{samsonov2024bootstrap} proposed a multiplier bootstrap approach
for constructing confidence sets.  These methods typically achieve valid
inference asymptotically but may have different finite-sample convergence
properties than direct covariance estimation.

\paragraph{Concentration inequalities for dependent data.}
Our variance analysis relies on matrix concentration inequalities
for weakly dependent sequences.  The matrix Bernstein inequality of
\citet{tropp2015introduction} provides sharp bounds for independent
summands; for the $\beta$-mixing block statistics arising in the
batch-means estimator, we appeal to the Marton coupling approach of
\citet{paulin2015concentration}, which extends scalar and matrix
concentration bounds to geometrically ergodic Markov chains with
explicit dependence on the mixing coefficients.

\paragraph{The i.i.d.\ baseline.}
For context, in the classical i.i.d.\ setting where one observes
$X_1,\ldots,X_n\sim\mathcal{N}(0,\Sigma)$, the sample covariance
$\widehat{\Sigma}=n^{-1}\sum_i X_i X_i^{\top}$ achieves the minimax-optimal
rate $\opnorm{\widehat{\Sigma}-\Sigma}=O(\sqrt{d/n})$ with high probability
\citep{vershynin2018}.  The $n^{-1/2}$ rate (for fixed $d$) is thus
the gold standard.  Our result shows that the minimax rate without
Hessian access in the SGD setting is $\Theta(n^{-(1-\alpha)/2})$
($n^{-1/4}$ at $\alpha\to 1/2^+$), achieved by the trajectory-regression
estimator---slower by a factor of $n^{1/4}$
than the i.i.d.\ optimum, reflecting the price of non-stationarity
and dependent data.

\section{Discussion}\label{sec:discussion}

\subsection{Why the Rate Cannot Be \texorpdfstring{$n^{-1/2}$}{n\textasciicircum(-1/2)}
Without Hessian Access}\label{sec:barriers}

It is natural to ask whether the batch-means rate $n^{-(1-\alpha)/3}$
can be improved.  The error analysis of Section~\ref{sec:error}
identifies the fundamental bottleneck for batch-means: the
mixing constraint on $\beta$ (Section~\ref{sec:mixing}) limits the
number of blocks, while the per-block bias $O(\tau_m/a_m)$
(Lemma~\ref{lem:bias}) limits how small the blocks can be.
The trajectory-regression estimator of
Section~\ref{sec:regression} bypasses both constraints.

\paragraph{What regression buys.}
The trajectory-regression estimator directly estimates $H$ by
regressing the SGD increments $y_t = (x_t-x_{t+1})/\eta_t$ on the
iterates~$x_t$.  This reduces the problem to a linear regression
with $\Theta(n^{1-\alpha})$ effective observations, achieving
$O(n^{-(1-\alpha)/2})$.  Methods using explicit Hessian
access---such as \citet{chen2020statistical}---achieve even faster
rates ($O(n^{-\alpha/2})$ approaching $n^{-1/2}$) but require
second-order oracle access.

\subsection{Minimax Optimality}\label{sec:optimality}

Theorems~\ref{thm:lower} and~\ref{thm:regression} together establish
the minimax rate $\Theta(n^{-(1-\alpha)/2})$ for Hessian-free
covariance estimation from the SGD trajectory.
The key mechanism behind the lower bound is that information
about $H$ accumulates at rate $\Theta(n^{1-\alpha})$, which is
strictly sublinear for $\alpha>0$.  The trajectory-regression
estimator is minimax-optimal; the batch-means estimator achieves the
slower rate $O(n^{-(1-\alpha)/3})$ but has practical advantages
(Remark~\ref{rem:bm-vs-regression}).

\subsection{Implications for Confidence Interval Coverage}
\label{sec:coverage}

The primary motivation for estimating $V$ is the construction of
confidence regions for $x^*$.  By the CLT \eqref{eq:clt}, an
asymptotically valid $(1-\nu)$-confidence ellipsoid is
\[
\left\{x\in\R^d :
n\,(\xbar - x)^{\top}\Sighat(\rho)^{-1}(\xbar - x)
\;\le\; \chi^2_{d,1-\nu}\right\},
\]
where $\chi^2_{d,1-\nu}$ is the $(1-\nu)$-quantile of the
chi-squared distribution with $d$ degrees of freedom.  The coverage
error of this ellipsoid depends on $\opnorm{\Sighat(\rho)-V}$ through
the perturbation of the quadratic form.  A standard argument (see,
e.g., \citealt{chen2020statistical}) gives a coverage error bound of
order $O(\opnorm{\widehat{V}-V}/\lambda_{\min}(V))$ (where $\widehat{V}$
is any covariance estimator), so the minimax
rate $n^{-(1-\alpha)/2}$ for covariance estimation (achieved by the
trajectory-regression estimator) translates directly into faster
convergence of coverage to the nominal level compared to the batch-means
rates $n^{-1/8}$ or $n^{-1/6}$.  Making this precise---establishing Berry--Esseen-type
bounds for the coverage probability---requires additional work on the
joint distribution of $\xbar$ and $\Sighat(\rho)$ and is left for
future investigation.

\section{Numerical Experiments}\label{sec:experiments}

We present numerical experiments that validate the theoretical results
and illustrate the practical benefits of the burn-in
modification.\footnote{Code to reproduce all experiments is available
at \url{https://github.com/Yijin911/online-cov-sgd}.}

\subsection{Quadratic Objective}\label{sec:exp-quadratic}

We consider the quadratic objective
\begin{equation}\label{eq:quadratic}
F(x) = \tfrac{1}{2}x^{\top}Hx,
\end{equation}
with $x\in\R^{10}$ ($d=10$).  The Hessian $H$ has eigenvalues equally
spaced from $1$ to $5$.  The stochastic gradient is
$\nabla f(x,\zeta) = Hx + \zeta$ with $\zeta\sim\mathcal{N}(0,I)$, so
$S=I$ and the true asymptotic covariance is $V=H^{-2}$.

We run SGD with step sizes $\eta_t=\eta_0 t^{-\alpha}$ and $\eta_0=1$
for $\alpha\in\{0.55,\,0.60,\,0.70\}$.  The sample size $n$ ranges from
$10^{3}$ to $10^{7}$.  For each configuration, we perform $200$
independent replications (reduced to $50$ for $n\ge 3\times 10^6$) and
report the average operator-norm error
$\E[\opnorm{\Sighat(\rho) - V}]$.

We use the growing-block construction $a_m=\lfloor C\,m^{\beta}\rfloor$
with $C=5$ as in the theoretical analysis.  We compare four estimators:
\begin{enumerate}[(i)]
\item \textbf{BM original} ($\rho=1$, $\beta^*=2/(1-\alpha)$): the
  estimator of \citet{zhu2023online}.
\item \textbf{BM optimal} ($\rho=1$,
  $\beta^\dagger=(1+2\alpha)/(2(1-\alpha))$):
  batch-means with the rate-optimal block growth of
  Corollary~\ref{cor:main}.
\item \textbf{BM burn-in} ($\rho=0.5$, $\beta^\dagger$):
  the optimal batch-means with burn-in.
\item \textbf{Trajectory regression}: the minimax-optimal estimator of
  Theorem~\ref{thm:regression}.
\end{enumerate}

\paragraph{Results.}
Figure~\ref{fig:rates} displays the operator-norm error versus $n$ on
log-log axes.  Three features stand out.

\emph{Clear ordering.}
The trajectory-regression estimator consistently achieves the lowest
error, followed by batch-means with burn-in, batch-means with optimal
$\beta^\dagger$ (no burn-in), and the original batch-means.

\emph{Convergence rates.}
Dotted reference lines show the theoretical slopes
$n^{-(1-\alpha)/4}$ (BM original),
$n^{-(1-\alpha)/3}$ (BM optimal), and
$n^{-(1-\alpha)/2}$ (regression).
The burn-in and regression empirical slopes match theory well at
$\alpha=0.55$ and $\alpha=0.6$.  The BM optimal without burn-in
($\rho=1$) shows steeper-than-asymptotic slopes because the
large early-block bias (constants of order 5--10) is still being
diluted at moderate $n$.

\emph{Practical value of burn-in.}
Although burn-in does not change the asymptotic batch-means rate
(Corollary~\ref{cor:main}), it removes the large-constant early-block
bias, yielding an order-of-magnitude error reduction at all sample
sizes tested.  This confirms that burn-in is critical in practice.

\begin{figure}[tbp]
\centering
\includegraphics[width=\textwidth]{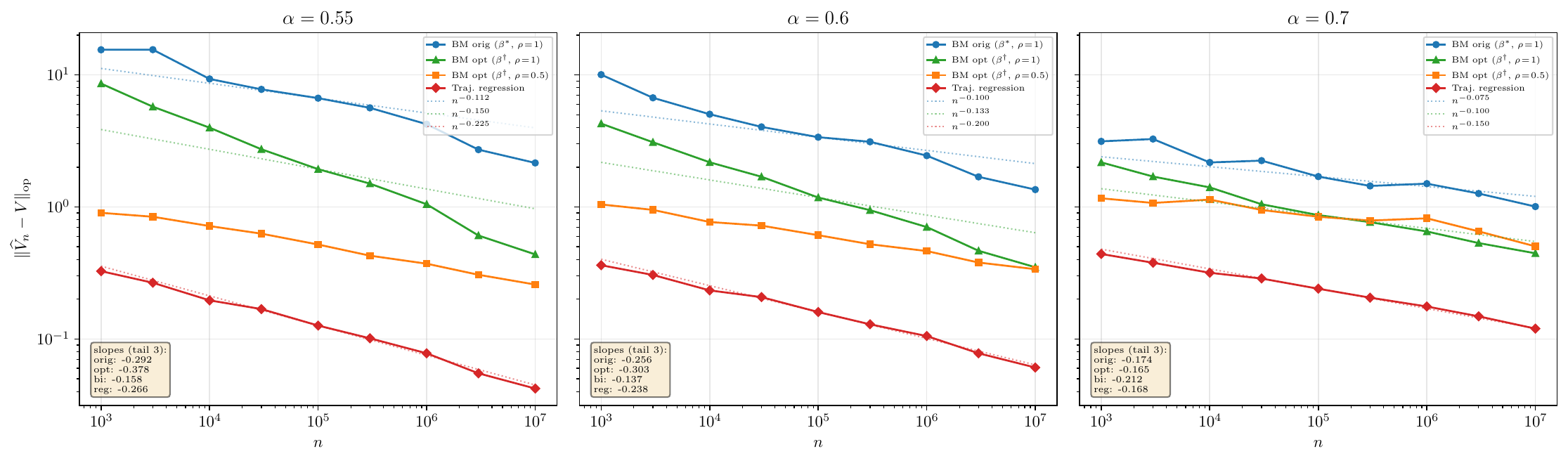}
\caption{Operator-norm error versus sample size $n$ (log-log scale)
for $\alpha\in\{0.55,0.60,0.70\}$.
Blue circles: BM original ($\beta^*$, $\rho=1$).
Green triangles: BM optimal ($\beta^\dagger$, $\rho=1$).
Orange squares: BM burn-in ($\beta^\dagger$, $\rho=0.5$).
Red diamonds: trajectory regression.
Dotted lines: theoretical reference slopes $n^{-(1-\alpha)/4}$,
$n^{-(1-\alpha)/3}$, and $n^{-(1-\alpha)/2}$.}
\label{fig:rates}
\end{figure}

\begin{remark}[Choice of $\beta$ in practice]
\label{rem:growing-vs-equal}
The theoretically optimal $\beta^\dagger=(1+2\alpha)/(2(1-\alpha))$
balances the variance $O(n^{-1/(2(\beta+1))})$ against the stationarity
bias $O(n^{-\gamma/(\beta+1)})$.  In finite samples, $\beta$ near the
mixing boundary yields blocks with few mixing times, so the within-block
CLT approximation is poor and the theoretical rate is masked by large
pre-asymptotic constants.  The practical remedy is to use a larger
$\beta$---such as $\beta=2\alpha/(1-\alpha)$---which ensures many mixing
times per block.  Burn-in ($\rho<1$) further helps in finite samples by
removing the highest-bias early blocks, even though it does not change the
asymptotic rate.  In offline settings (where $n$ is known), equal-sized
blocks can further mitigate this effect.
\end{remark}

\subsection{Bias--Variance Decomposition}\label{sec:exp-bv}

To illustrate the mechanism behind the improvement, we decompose the
per-block error into bias and variance components.  We fix $\alpha=0.55$
and $n=10^{5}$, and compute for each block $m$:
\begin{itemize}
\item \textbf{Bias:} $\opnorm{\E[Y_m Y_m^{\top}] - V}$, estimated by
  averaging $Y_m Y_m^{\top}$ over $500$ replications.
\item \textbf{Variance:} $\E[\opnorm{Y_m Y_m^{\top} -
  \E[Y_m Y_m^{\top}]}]$, estimated similarly.
\end{itemize}

\paragraph{Results.}
Figure~\ref{fig:bv} shows the per-block bias and variance as functions
of the block index $m$.  The bias is large for early blocks
($m=1,2,\ldots$) and decreases as $m$ grows, reflecting the chain's
approach to stationarity.  The variance is roughly constant across blocks
(slightly decreasing due to larger block sizes).  The crossover point---where
bias drops below variance---occurs at approximately $m=b_n/3$.  Blocks
before this point are bias-dominated, and their inclusion in the average
degrades the overall estimator.  The burn-in strategy discards precisely
these blocks.

\begin{figure}[tbp]
\centering
\includegraphics[width=0.85\textwidth]{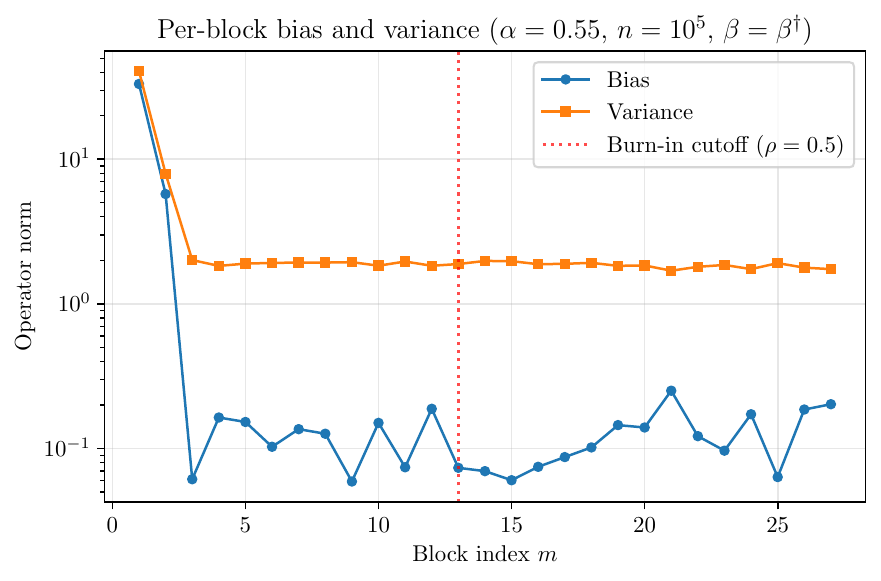}
\caption{Per-block bias ($\circ$, solid) and variance ($\square$, solid)
as functions of block index $m$, for $\alpha=0.55$ and $n=10^{5}$.
Early blocks have large bias and contribute disproportionately to the
estimation error.  The vertical dotted line marks the burn-in cutoff
($\rho=0.5$); blocks to its right are retained.}
\label{fig:bv}
\end{figure}

\section{Conclusion}\label{sec:conclusion}

We have studied online covariance estimation for Polyak--Ruppert
averaged SGD with two main contributions.  First, a rigorous per-block
bias analysis (Lemma~\ref{lem:bias}) shows that the stationarity bias
is $O(\tau_m/a_m)$---the ratio of mixing time to block size---which
leads to an improved batch-means rate of $O(n^{-(1-\alpha)/3})$ via
optimal block-growth tuning (Corollary~\ref{cor:main}), up from the
$O(n^{-(1-\alpha)/4})$ rate of \citet{zhu2023online}.  Burn-in
($\rho<1$) improves finite-sample constants but does not change the
asymptotic batch-means rate.  Second, we establish the minimax rate
$\Theta(n^{-(1-\alpha)/2})$ for Hessian-free covariance estimation from
the SGD trajectory: a Le~Cam lower bound (Theorem~\ref{thm:lower})
gives $\Omega(n^{-(1-\alpha)/2})$, and a trajectory-regression
estimator (Theorem~\ref{thm:regression}) achieves $O(n^{-(1-\alpha)/2})$
by regressing SGD increments on iterates to implicitly estimate the
Hessian.

Several directions remain open:
\begin{enumerate}[(1)]
\item \textbf{Closing the batch-means gap.}
  The batch-means rate $O(n^{-(1-\alpha)/3})$ is suboptimal compared
  to the minimax rate $\Theta(n^{-(1-\alpha)/2})$.  Whether this gap
  is inherent to block-based estimators or can be closed by a more
  refined batch-means analysis is an interesting open question.
\item \textbf{Markovian noise.}
  \citet{roy2023online} extended the original batch-means analysis
  to Markovian data.  It would be valuable to verify that the per-block
  bias analysis and trajectory-regression approach extend to that
  setting.
\item \textbf{Berry--Esseen bounds.}
  For practical confidence interval construction, one needs not only
  consistency of the covariance estimator but also Berry--Esseen-type
  bounds on the coverage error.  The minimax rate $n^{-(1-\alpha)/2}$
  should translate into improved coverage guarantees, but making this
  precise requires additional work.
\end{enumerate}

\begin{acks}[Acknowledgments]
Ni is supported by the Robert Goodell Ph.D.\ Student Fellowship for Research Excellence at Georgia Institute of Technology.
Huo is partially supported by a subcontract of NSF grant 2229876,
the A.\ Russell Chandler III Professorship at Georgia Institute of
Technology, an NIH-sponsored Georgia Clinical \& Translational
Science Alliance, and the Georgia Department of Transportation.
\end{acks}

\begin{supplement}
\stitle{Supplement to ``Online Covariance Estimation in Averaged SGD:
Improved Batch-Means Rates and Minimax Optimality via
Trajectory Regression''}
\sdescription{The supplementary material reviews the classical i.i.d.\
operator-norm convergence rate for covariance estimation, providing
context for the SGD-specific rates derived in the main paper.
Simulation code is available at
\url{https://github.com/Yijin911/online-cov-sgd}.}
\end{supplement}

\bibliographystyle{imsart-nameyear}
\bibliography{references}

\end{document}